\documentclass[a4paper,12pt]{amsart}

\newtheorem{theorem}{Theorem}
\newtheorem{lemma}{Lemma}
\newtheorem{corollary}{Corollary}
\newtheorem{definition}{Definition}
\newtheorem{proposition}{Proposition}
\theoremstyle{remark}
\newtheorem{remark}{Remark}
\theoremstyle{example}
\newtheorem{example}{Example}

\usepackage{graphicx,natbib,amssymb,datetime,float,MnSymbol,currfile}
\usepackage{tikz}
\usetikzlibrary{arrows}

\def\ci{\!\perp\!}
\def\nci{\!\not\perp\!}
\def\ra{\rightarrow}
\def\la{\leftarrow}
\def\aa{\leftrightarrow}
\def\bb{\leftfootline\!\!\!\!\!\rightfootline}
\def\bo{\leftfootline\!\!\!\!\!\multimap}
\def\ao{\leftarrow\!\!\!\!\!\multimap}
\def\ob{\mathrel{\reflectbox{\ensuremath{\bo}}}}
\def\oa{\mathrel{\reflectbox{\ensuremath{\ao}}}}
\def\bn{\leftfootline}
\def\nb{\rightfootline}
\def\oo{\mathrel{\reflectbox{\ensuremath{\multimap}}}\!\!\!\!\!\multimap}
\def\no{\multimap}
\newcommand{\comments}[1]{}
\def\de{\!\sim\!}

\begin{document}

\title[]{Learning AMP Chain Graphs and some Marginal Models Thereof under Faithfulness: Extended Version}

\author[]{Jose M. Pe\~{n}a\\
ADIT, IDA, Link\"oping University, SE-58183 Link\"{o}ping, Sweden\\
jose.m.pena@liu.se
}

\begin{abstract}
This paper deals with chain graphs under the Andersson-Madigan-Perlman (AMP) interpretation. In particular, we present a constraint based algorithm for learning an AMP chain graph a given probability distribution is faithful to. Moreover, we show that the extension of Meek's conjecture to AMP chain graphs does not hold, which compromises the development of efficient and correct score+search learning algorithms under assumptions weaker than faithfulness. 

We also study the problem of how to represent the result of marginalizing out some nodes in an AMP CG. We introduce a new family of graphical models that solves this problem partially. We name this new family maximal covariance-concentration graphs (MCCGs) because it includes both covariance and concentration graphs as subfamilies. We describe global, local and pairwise Markov properties for MCCGs and prove their equivalence. We characterize when two MCCGs are Markov equivalent, and show that every Markov equivalence class of MCCGs has a distinguished member. We present a constraint based algorithm for learning a MCCG a given probability distribution is faithful to. 

Finally, we present a graphical criterion for reading dependencies from a MCCG of a probability distribution that satisfies the graphoid properties, weak transitivity and composition. We prove that the criterion is sound and complete in certain sense.
\end{abstract}

\date{\currfilename, \currenttime, \ddmmyydate{\today}}

\maketitle

\section{Introduction}\label{sec:introduction}

This paper deals with chain graphs (CGs) under the Andersson-Madigan-Perlman (AMP) interpretation \citep{Anderssonetal.2001}. Two other interpretations exist in the literature, namely the Lauritzen-Wermuth-Frydenberg (LWF) interpretation \citep{Lauritzen1996} and the multivariate regression (MVR) interpretation \citep{CoxandWermuth1996}. The AMP and LWF interpretations are sometimes considered as competing and, thus, their relative merits have been pointed out \citep{Anderssonetal.2001,DrtonandEichler2006,Levitzetal.2001,RoveratoandStudeny2006}. Note, however, that no interpretation subsumes the other: There are many independence models that can be induced by a CG under one interpretation but that cannot be induced by any CG under the other interpretation \citep[Theorem 6]{Anderssonetal.2001}. Likewise, neither the AMP interpretation subsumes the MVR interpretation nor vice versa \citep[Theorems 4 and 5]{SonntagandPenna2013}.

This paper consists of three main sections. In Section \ref{sec:amp}, we present an algorithm for learning an AMP CG a given probability distribution is faithful to. To our knowledge, we are the first to present such an algorithm. However, algorithms for learning LWF CGs under faithfulness already exist \citep{Maetal.2008,Studeny1997a}. In fact, we have recently developed an algorithm for learning LWF CGs under the milder composition property assumption \citep{Pennaetal.2012}. We have also recently developed an algorithm for learning MVR CGs under the faithfulness assumption \citep{SonntagandPenna2012}.

As \citet[Section 9.4]{RichardsonandSpirtes2002} show, a desirable feature that AMP CGs lack is that of being closed under marginalization (a.k.a the precollapsibility property \citep{Studeny1997b}). That is, the independence model resulting from marginalizing out some nodes in an AMP CG may not be represented by any other AMP CG. This leads us to the problem of how to represent the result of marginalizing out some nodes in an AMP CG. Of course, one may decide to continue working with the AMP CG and treat the marginalized nodes as latent nodes. This solution relies upon one having access to the AMP CG. Thus, it does not solve the problem if one knows that there is an underlying AMP CG but does not have access to it. As far as we know, this problem has been studied for directed and acyclic graphs by \cite{RichardsonandSpirtes2002} but not for AMP CGs. In Section \ref{sec:mccg}, we present the partial solution to this problem that we have obtained so far. Specifically, we introduce and study a new family of graphical models that we call maximal covariance-concentration graphs (MCCGs). MCCGs solve the problem at hand partially, because each of them represents the result of marginalizing out some nodes in some AMP CG. Unfortunately, MCCGs do not solve the problem completely, because they do not represent the result of marginalizing out any nodes in any AMP CG.

MCCGs consist of undirected and bidirected edges, and they unify and generalize covariance and concentration graphs, hence the name. Concentration graphs (a.k.a Markov networks) were introduced by \cite{Pearl1988} to represent independence models. Specifically, the concentration graph of a probability distribution $p$ is the undirected graph $G$ where two nodes are not adjacent if and only if their corresponding random variables are independent in $p$ given the rest of the random variables. Graphical criteria for reading dependencies and independencies from $G$ (under certain assumptions about $p$) have been proposed \citep{Bouckaert1995,Pearl1988,Pennaetal.2009}. Likewise, covariance graphs (a.k.a bidirected graphs) were introduced by \cite{CoxandWermuth1996} to represent independence models. Specifically, the covariance graph of a probability distribution $p$ is the bidirected graph $G$ where two nodes are not adjacent if and only if their corresponding random variables are marginally independent in $p$. Graphical criteria for reading dependencies and independencies from $G$ (under certain assumptions about $p$) have been proposed \citep{BanerjeeandRichardson2003,Kauermann1996,Penna2013}.

If we focus on Gaussian probability distributions, then one could say that the covariance graph of a Gaussian probability distribution models its covariance matrix, whereas its concentration graph models its concentration matrix. We think that Gaussian probability distributions would be modeled more accurately if their covariance and concentration matrices were modeled jointly by a single graph. This is something one can do with MCCGs.

Finally, in Section \ref{sec:dependencies} we present a graphical criterion for reading dependencies from a MCCG $G$ of a probability distribution $p$, under the assumption that $G$ satisfies some topological constraints and $p$ satisfies the graphoid properties, weak transitivity and composition. We prove that the graphical criterion is sound and complete in certain sense.

\section{Preliminaries}\label{sec:preliminaries}

In this section, we review some concepts from probabilistic graphical models that are used later in this paper. All the graphs and probability distributions in this paper are defined over a finite set $V$. All the graphs in this paper are simple, i.e. they contain at most one edge between any pair of nodes. The elements of $V$ are not distinguished from singletons. We denote by $|X|$ the cardinality of $X \subseteq V$.

If a graph $G$ contains an undirected, directed or bidirected edge between two nodes $V_{1}$ and $V_{2}$, then we write that $V_{1} - V_{2}$, $V_{1} \ra V_{2}$ or $V_{1} \aa V_{2}$ is in $G$. The parents of a set of nodes $X$ of $G$ is the set $pa_G(X) = \{V_1 | V_1 \ra V_2$ is in $G$, $V_1 \notin X$ and $V_2 \in X \}$. The neighbors of a set of nodes $X$ of $G$ is the set $ne_G(X) = \{V_1 | V_1 - V_2$ is in $G$, $V_1 \notin X$ and $V_2 \in X \}$. The spouses of a set of nodes $X$ of $G$ is the set $sp_G(X) = \{V_1 | V_1 \aa V_2$ is in $G$, $V_1 \notin X$ and $V_2 \in X \}$. The adjacents of a set of nodes $X$ of $G$ is the set $ad_G(X) = \{V_1 | V_1 \ra V_2$, $V_1 - V_2$ or $V_1 \la V_2$ is in $G$, $V_1 \notin X$ and $V_2 \in X \}$. A route from a node $V_{1}$ to a node $V_{n}$ in $G$ is a sequence of (not necessarily distinct) nodes $V_{1}, \ldots, V_{n}$ such that $V_i \in ad_G(V_{i+1})$ for all $1 \leq i < n$. If the nodes in the route are all distinct, then the route is called a path. The length of a route is the number of (not necessarily distinct) edges in the route, e.g. the length of the route $V_{1}, \ldots, V_{n}$ is $n-1$. A route is called a cycle if $V_n=V_1$. A cycle has a chord if two non-consecutive nodes of the cycle are adjacent in $G$. A route is called descending if $V_{i} \in pa_G(V_{i+1}) \cup ne_G(V_{i+1})$ for all $1 \leq i < n$. The descendants of a set of nodes $X$ of $G$ is the set $de_G(X) = \{V_n |$ there is a descending route from $V_1$ to $V_n$ in $G$, $V_1 \in X$ and $V_n \notin X \}$. A cycle is called a semidirected cycle if it is descending and $V_{i} \ra V_{i+1}$ is in $G$ for some $1 \leq i < n$. A chain graph (CG) is a graph whose every edge is undirected or directed, and that has no semidirected cycles. A set of nodes of a graph is complete if there is an undirected edge between every pair of nodes in the set. A set of nodes of a graph is undirectly (respectively bidirectly) connected if there exists a route in the graph between every pair of nodes in the set such that all the edges in the route are undirected (respectively bidirected). An undirected (respectively bidirected) connectivity component of a graph is an undirectly (respectively bidirectly) connected set that is maximal (with respect to set inclusion). The undirected connectivity component a node $A$ of a graph $G$ belongs to is denoted as $co_G(A)$. The subgraph of $G$ induced by a set of its nodes $X$, denoted as $G_X$, is the graph over $X$ that has all and only the edges in $G$ whose both ends are in $X$. An immorality in a CG is an induced subgraph of the form $A \ra B \la C$. A flag in a CG is an induced subgraph of the form $A \ra B - C$. If a CG $G$ has an induced subgraph of the form $A \ra B \la C$, $A \ra B - C$ or $A - B \la C$, then we say that the triplex $(\{A,C\},B)$ is in $G$. Two CGs are triplex equivalent if and only if they have the same adjacencies and the same triplexes.

Let $X$, $Y$, $Z$ and $W$ denote four pairwise disjoint subsets of $V$. An independence model $M$ is a set of statements $X \ci_M Y | Z$. $M$ satisfies the graphoid properties if it satisfies the following properties: 

\begin{itemize}
\item Symmetry $X \ci_M Y | Z \Rightarrow Y \ci_M X | Z$.

\item Decomposition $X \ci_M Y \cup W | Z \Rightarrow X \ci_M Y | Z$.

\item Weak union $X \ci_M Y \cup W | Z \Rightarrow X \ci_M Y | Z \cup W$.

\item Contraction $X \ci_M Y | Z \cup W \land X \ci_M W | Z \Rightarrow X \ci_M Y \cup W | Z$.

\item Intersection $X \ci_M Y | Z \cup W \land X \ci_M W | Z \cup Y \Rightarrow X \ci_M Y \cup W | Z$. 
\end{itemize}

Two other properties that $M$ may satisfy are the following:

\begin{itemize}
\item Composition $X \ci_M Y | Z \land X \ci_M W | Z \Rightarrow X \ci_M Y \cup W | Z$.

\item Weak transitivity $X \ci_M Y | Z \land X \ci_M Y | Z \cup K \Rightarrow X \ci_M K | Z \lor K \ci_M Y | Z$ with $K \in V \setminus X \setminus Y \setminus Z$. 
\end{itemize}

We say that an independence model is a WTC graphoid when it satisfies the seven previous properties. We denote by $X \ci_p Y | Z$ (respectively $X \nci_p$ $Y | Z$) that $X$ is independent (respectively dependent) of $Y$ given $Z$ in a probability distribution $p$. We say that $p$ is Markovian with respect to an independence model $M$ when $X \ci_p Y | Z$ if $X \ci_M Y | Z$ for all $X$, $Y$ and $Z$ pairwise disjoint subsets of $V$. We say that $p$ is faithful to $M$ when $X \ci_p Y | Z$ if and only if $X \ci_M Y | Z$ for all $X$, $Y$ and $Z$ pairwise disjoint subsets of $V$. Any probability distribution $p$ satisfies the first four previous properties. If $p$ is faithful to a CG, then it also satisfies the last three previous properties.\footnote{To see it, note that there is a Gaussian distribution that is faithful to $G$ \citep[Theorem 6.1]{Levitzetal.2001}. Moreover, every Gaussian distribution satisfies the intersection, composition and weak transitivity properties \citep[Proposition 2.1 and Corollaries 2.4 and 2.5]{Studeny2005}.}

A node $B$ in a route $\rho$ in a CG is called a head-no-tail node in $\rho$ if $A \ra B \la C$, $A \ra B - C$, or $A - B \la C$ is a subroute of $\rho$ (note that maybe $A=C$ in the first case). A node $B$ in $\rho$ is called a non-head-no-tail node in $\rho$ if $A \la B \ra C$, $A \la B \la C$, $A \la B - C$, $A \ra B \ra C$, $A - B \ra C$, or $A - B - C$ is a subroute of $\rho$ (note that maybe $A=C$ in the first and last cases). Note that to classify $B$ as a (non-)head-no-tail node in $\rho$, one has to consider the edge ends at $B$ as well as at $A$ and $C$. Note also that $B$ may be both a head-no-tail and a non-head-no-tail node in $\rho$, e.g. take $\rho$ to be $A \ra B \la C \ra B \ra D$. Let $X$, $Y$ and $Z$ denote three pairwise disjoint subsets of $V$. A route $\rho$ in a CG $G$ is said to be $Z$-open when (i) every head-no-tail node in $\rho$ is in $Z$, and (ii) every non-head-no-tail node in $\rho$ is not in $Z$.\footnote{Note that if a node is both a head-no-tail and a non-head-no-tail node in $\rho$, then $\rho$ is not $Z$-open.} When there is no route in $G$ between a node in $X$ and a node in $Y$ that is $Z$-open, we say that $X$ is separated from $Y$ given $Z$ in $G$ and denote it as $X \ci_G Y | Z$.\footnote{See \citep[Remark 3.1]{Anderssonetal.2001} for the equivalence of this and the standard definition of separation.} We denote by $X \nci_G Y | Z$ that $X \ci_G Y | Z$ does not hold. The independence model induced by $G$, denoted as $I(G)$, is the set of separation statements $X \ci_G$ $Y | Z$. If two CGs $G$ and $H$ are triplex equivalent, then $I(G)=I(H)$.\footnote{To see it, note that there are Gaussian distributions $p$ and $q$ that are faithful to $G$ and $H$, respectively \citep[Theorem 6.1]{Levitzetal.2001}. Moreover, $p$ and $q$ are Markovian with respect to $H$ and $G$, respectively, by \citet[Theorem 5]{Anderssonetal.2001} and \citet[Theorem 4.1]{Levitzetal.2001}.}

\section{Algorithm for Learning AMP CGs}\label{sec:amp}

In this section, we present an algorithm for learning an AMP CG a given probability distribution is faithful to. The algorithm, which can be seen in Table \ref{tab:algorithm}, resembles the well-known PC algorithm \citep{Meek1995,Spirtesetal.1993}. It consists of two phases: The first phase (lines 1-8) aims at learning adjacencies, whereas the second phase (lines 9-10) aims at directing some of the adjacencies learnt. Specifically, the first phase declares that two nodes are adjacent if and only if they are not separated by any set of nodes. Note that the algorithm does not test every possible separator (see line 5). Note also that the separators tested are tested in increasing order of size (see lines 2, 5 and 8). The second phase consists of two steps. In the first step, the ends of some of the edges learnt in the first phase are blocked according to the rules R1-R4 in Table \ref{tab:rules}. A block is represented by a perpendicular line such as in $\bn$ or $\bb$, and it means that the edge cannot be directed in that direction. In the second step, the edges with exactly one unblocked end get directed in the direction of the unblocked end. The rules R1-R4 work as follows: If the conditions in the antecedent of a rule are satisfied, then the modifications in the consequent of the rule are applied. Note that the ends of some of the edges in the rules are labeled with a circle such as in $\bo$ or $\oo$. The circle represents an unspecified end, i.e. a block or nothing. The modifications in the consequents of the rules consist in adding some blocks. Note that only the blocks that appear in the consequents are added, i.e. the circled ends do not get modified. The conditions in the antecedents of R1, R2 and R4 consist of an induced subgraph of $H$ and the fact that some of its nodes are or are not in some separators found in line 6. The condition in the antecedent of R3 consists of just an induced subgraph of $H$. Specifically, the antecedent says that there is a cycle in $H$ whose edges have certain blocks. Note that the cycle must be chordless.

\begin{table}[t]
\caption{Algorithm for learning AMP CGs.}\label{tab:algorithm}
\centering
\scalebox{0.8}{
\begin{tabular}{rl}
\\
\hline
\\
& Input: A probability distribution $p$ that is faithful to an unknown CG $G$.\\
& Output: A CG $H$ that is triplex equivalent to $G$.\\
\\
1 & Let $H$ denote the complete undirected graph\\
2 & Set $l=0$\\
3 & Repeat while $l \leq |V|-2$\\
4 & \hspace{0.2cm} For each ordered pair of nodes $A$ and $B$ in $H$ st $A \in ad_H(B)$ and $|[ad_H(A) \cup ad_H(ad_H(A)) ] \setminus B| \geq l$\\
5 & \hspace{0.5cm} If there is some $S \subseteq [ad_H(A) \cup ad_H(ad_H(A)) ] \setminus B$ such that $|S|=l$ and $A \ci_p B | S$ then\\
6 & \hspace{0.8cm} Set $S_{AB}=S_{BA}=S$\\
7 & \hspace{0.8cm} Remove the edge $A - B$ from $H$\\
8 & \hspace{0.2cm} Set $l=l+1$\\
9 & Apply the rules R1-R4 to $H$ while possible\\
10 & Replace every edge $A \bn B$ (respectively $A \bb B$) in $H$ with $A \ra B$ (respectively $A - B$)\\
\\
\hline
\\
\end{tabular}}
\end{table}

\begin{table}[t]
\caption{Rules R1-R4 in the algorithm for learning AMP CGs.}\label{tab:rules}
\centering
\scalebox{0.75}{
\begin{tabular}{cccc}
\\
\hline
\\
R1:&
\begin{tabular}{c}
\begin{tikzpicture}[inner sep=1mm]
\node at (0,0) (A) {$A$};
\node at (1.5,0) (B) {$B$};
\node at (3,0) (C) {$C$};
\path[o-o] (A) edge (B);
\path[o-o] (B) edge (C);
\end{tikzpicture} 
\end{tabular}
& $\Rightarrow$ &
\begin{tabular}{c}
\begin{tikzpicture}[inner sep=1mm]
\node at (0,0) (A) {$A$};
\node at (1.5,0) (B) {$B$};
\node at (3,0) (C) {$C$};
\path[|-o] (A) edge (B);
\path[o-|] (B) edge (C);
\end{tikzpicture}
\end{tabular}\\
& $\land$ $B \notin S_{AC}$\\
\\
\hline
\\
R2:&
\begin{tabular}{c}
\begin{tikzpicture}[inner sep=1mm]
\node at (0,0) (A) {$A$};
\node at (1.5,0) (B) {$B$};
\node at (3,0) (C) {$C$};
\path[|-o] (A) edge (B);
\path[o-o] (B) edge (C);
\end{tikzpicture}
\end{tabular}
& $\Rightarrow$ &
\begin{tabular}{c}
\begin{tikzpicture}[inner sep=1mm]
\node at (0,0) (A) {$A$};
\node at (1.5,0) (B) {$B$};
\node at (3,0) (C) {$C$};
\path[|-o] (A) edge (B);
\path[|-o] (B) edge (C);
\end{tikzpicture}
\end{tabular}\\
& $\land$ $B \in S_{AC}$\\
\\
\hline
\\
R3:&
\begin{tabular}{c}
\begin{tikzpicture}[inner sep=1mm]
\node at (0,0) (A) {$A$};
\node at (1.5,0) (B) {$\ldots$};
\node at (3,0) (C) {$B$};
\path[|-o] (A) edge (B);
\path[|-o] (B) edge (C);
\path[o-o] (A) edge [bend left] (C);
\end{tikzpicture}
\end{tabular}
& $\Rightarrow$ &
\begin{tabular}{c}
\begin{tikzpicture}[inner sep=1mm]
\node at (0,0) (A) {$A$};
\node at (1.5,0) (B) {$\ldots$};
\node at (3,0) (C) {$B$};
\path[|-o] (A) edge (B);
\path[|-o] (B) edge (C);
\path[|-o] (A) edge [bend left] (C);
\end{tikzpicture}
\end{tabular}\\
\\
\hline
\\
R4:&
\begin{tabular}{c}
\begin{tikzpicture}[inner sep=1mm]
\node at (0,0) (A) {$A$};
\node at (2,0) (B) {$B$};
\node at (1,1) (C) {$C$};
\node at (1,-1) (D) {$D$};
\path[o-o] (A) edge (B);
\path[o-o] (A) edge (C);
\path[o-o] (A) edge (D);
\path[|-o] (C) edge (B);
\path[|-o] (D) edge (B);
\end{tikzpicture}
\end{tabular}
& $\Rightarrow$ &
\begin{tabular}{c}
\begin{tikzpicture}[inner sep=1mm]
\node at (0,0) (A) {$A$};
\node at (2,0) (B) {$B$};
\node at (1,1) (C) {$C$};
\node at (1,-1) (D) {$D$};
\path[|-o] (A) edge (B);
\path[o-o] (A) edge (C);
\path[o-o] (A) edge (D);
\path[|-o] (C) edge (B);
\path[|-o] (D) edge (B);
\end{tikzpicture}
\end{tabular}\\
& $\land$ $A \in S_{CD}$\\
\\
\hline
\\
\end{tabular}}
\end{table}

\subsection{Correctness of the Algorithm}\label{sec:correctness}

In this section, we prove that our algorithm is correct, i.e. it returns a CG the given probability distribution is faithful to. We start proving a result for any probability distribution that satisfies the intersection and composition properties. Recall that any probability distribution that is faithful to a CG satisfies these properties and, thus, the following result applies to it.

\begin{lemma}\label{lem:conditions}
Let $p$ denote a probability distribution that satisfies the intersection and composition properties. Then, $p$ is Markovian with respect to a CG $G$ if and only if $p$ satisfies the following conditions:
\begin{itemize}

\item[C1:] $A \ci_p co_G(A) \setminus A \setminus ne_G(A) | pa_G(A \cup ne_G(A)) \cup ne_G(A)$ for all $A \in V$, and

\item[C2:] $A \ci_p V \setminus A \setminus de_G(A) \setminus pa_G(A) | pa_G(A)$ for all $A \in V$.

\end{itemize}

\end{lemma}

\begin{proof}
It follows from \citet[Theorem 3]{Anderssonetal.2001} and \citet[Theorem 4.1]{Levitzetal.2001} that $p$ is Markovian with respect to $G$ if and only if $p$ satisfies the following conditions:
\begin{itemize}

\item[L1:] $A \ci_p co_G(A) \setminus A \setminus ne_G(A) | [ V \setminus co_G(A) \setminus de_G(co_G(A)) ] \cup ne_G(A)$ for all $A \in V$, and

\item[L2:] $A \ci_p V \setminus co_G(A) \setminus de_G(co_G(A)) \setminus pa_G(A) | pa_G(A)$ for all $A \in V$.

\end{itemize}

Clearly, C2 holds if and only if L2 holds because $de_G(A) = [ co_G(A) \cup de_G(co_G(A)) ] \setminus A$. We prove below that if L2 holds, then C1 holds if and only if L1 holds. We first prove the if part.

\begin{itemize}

\item[1.] $B \ci_p V \setminus co_G(B) \setminus de_G(co_G(B)) \setminus pa_G(B) | pa_G(B)$ for all $B \in A \cup ne_G(A)$ by L2.

\item[2.] $B \ci_p V \setminus co_G(B) \setminus de_G(co_G(B)) \setminus pa_G(A \cup ne_G(A)) | pa_G(A \cup ne_G(A))$ for all $B \in A \cup ne_G(A)$ by weak union on 1.

\item[3.] $A \cup ne_G(A) \ci_p V \setminus co_G(A) \setminus de_G(co_G(A)) \setminus pa_G(A \cup ne_G(A)) | pa_G(A \cup ne_G(A))$ by repeated application of symmetry and composition on 2.

\item[4.] $A \ci_p V \setminus co_G(A) \setminus de_G(co_G(A)) \setminus pa_G(A \cup ne_G(A)) | pa_G(A \cup ne_G(A)) \cup ne_G(A)$ by symmetry and weak union on 3.

\item[5.] $A \ci_p co_G(A) \setminus A \setminus ne_G(A) | [ V \setminus co_G(A) \setminus de_G(co_G(A)) ] \cup ne_G(A)$ by L1.

\item[6.] $A \ci_p [ co_G(A) \setminus A \setminus ne_G(A) ] \cup [ V \setminus co_G(A) \setminus de_G(co_G(A)) \setminus pa_G(A \cup ne_G(A)) ] | pa_G(A \cup ne_G(A)) \cup ne_G(A)$ by contraction on 4 and 5.

\item[7.] $A \ci_p co_G(A) \setminus A \setminus ne_G(A) | pa_G(A \cup ne_G(A)) \cup ne_G(A)$ by decomposition on 6.

\end{itemize}

We now prove the only if part.

\begin{itemize}

\item[8.] $A \ci_p co_G(A) \setminus A \setminus ne_G(A) | pa_G(A \cup ne_G(A)) \cup ne_G(A)$ by C1.

\item[9.] $A \ci_p [ V \setminus co_G(A) \setminus de_G(co_G(A)) \setminus pa_G(A \cup ne_G(A)) ] \cup [ co_G(A) \setminus A \setminus ne_G(A) ] | pa_G(A \cup ne_G(A)) \cup ne_G(A)$ by composition on 4 and 8.

\item[10.] $A \ci_p co_G(A) \setminus A \setminus ne_G(A) | [ V \setminus co_G(A) \setminus de_G(co_G(A)) ] \cup ne_G(A)$ by weak union on 9.

\end{itemize}

\end{proof}

\begin{lemma}\label{lem:adjacencies}
After line 8, $G$ and $H$ have the same adjacencies.
\end{lemma}

\begin{proof}
Consider any pair of nodes $A$ and $B$ in $G$. If $A \in ad_G(B)$, then $A \nci_p B | S$ for all $S \subseteq V \setminus [ A \cup B ]$ by the faithfulness assumption. Consequently, $A \in ad_H(B)$ at all times. On the other hand, if $A \notin ad_G(B)$, then consider the following cases. 

\begin{description}

\item[Case 1] Assume that $co_G(A)=co_G(B)$. Then, $A \ci_p co_G(A) \setminus A \setminus ne_G(A) | pa_G(A \cup ne_G(A)) \cup ne_G(A)$ by C1 in Lemma \ref{lem:conditions} and, thus, $A \ci_p B | pa_G(A \cup ne_G(A)) \cup ne_G(A)$ by decomposition and $B \notin ne_G(A)$, which follows from $A \notin ad_G(B)$. Note that, as shown above, $pa_G(A \cup ne_G(A)) \cup ne_G(A) \subseteq [ ad_H(A) \cup ad_H(ad_H(A)) ] \setminus B$ at all times.

\item[Case 2] Assume that $co_G(A) \neq co_G(B)$. Then, $A \notin de_G(B)$ or $B \notin de_G(A)$ because $G$ has no semidirected cycle. Assume without loss of generality that $B \notin de_G(A)$. Then, $A \ci_p$ $V \setminus A \setminus de_G(A) \setminus pa_G(A) | pa_G(A)$ by C2 in Lemma \ref{lem:conditions} and, thus, $A \ci_p B | pa_G(A)$ by decomposition, $B \notin de_G(A)$, and $B \notin pa_G(A)$ which follows from $A \notin ad_G(B)$. Note that, as shown above, $pa_G(A)\subseteq ad_H(A) \setminus B$ at all times.

\end{description}

Therefore, in either case, there will exist some $S$ in line 5 such that $A \ci_p B | S$ and, thus, the edge $A - B$ will be removed from $H$ in line 7. Consequently, $A \notin ad_H(B)$ after line 8. 
\end{proof}

The next lemma proves that the rules R1-R4 are sound in certain sense.

\begin{lemma}\label{lem:soundness}
The rules R1-R4 are sound in the sense that they block only those edge ends that are not arrowheads in $G$.
\end{lemma}

\begin{proof}
According to the antecedent of R1, $G$ has a triplex $(\{A,C\},B)$. Then, $G$ has an induced subgraph of the form $A \ra B \la C$, $A \ra B - C$ or $A - B \la C$. In either case, the consequent of R1 holds.

According to the antecedent of R2, (i) $G$ does not have a triplex $(\{A,C\},B)$, (ii) $A \ra B$ or $A - B$ is in $G$, (iii) $B \in ad_G(C)$, and (iv) $A \notin ad_G(C)$. Then, $B \ra C$ or $B - C$ is in $G$. In either case, the consequent of R2 holds.

According to the antecedent of R3, (i) $G$ has a descending route from $A$ to $B$, and (ii) $A \in ad_G(B)$. Then, $A \ra B$ or $A - B$ is in $G$, because $G$ has no semidirected cycle. In either case, the consequent of R3 holds. 

According to the antecedent of R4, neither $B \ra C$ nor $B \ra D$ are in $G$. Assume to the contrary that $A \la B$ is in $G$. Then, $G$ must have an induced subgraph that is consistent with

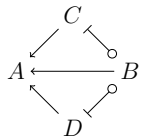
\begin{table}[H]
\centering
\scalebox{0.75}{
\begin{tikzpicture}[inner sep=1mm]
\node at (0,0) (A) {$A$};
\node at (2,0) (B) {$B$};
\node at (1,1) (C) {$C$};
\node at (1,-1) (D) {$D$};
\path[<-] (A) edge (B);
\path[<-] (A) edge (C);
\path[<-] (A) edge (D);
\path[|-o] (C) edge (B);
\path[|-o] (D) edge (B);
\end{tikzpicture}}
\end{table}

because, otherwise, it would have a semidirected cycle. However, this induced subgraph contradicts that $A \in S_{CD}$. 
\end{proof}

\begin{lemma}\label{lem:triplexes}
After line 10, $G$ and $H$ have the same triplexes. Moreover, $H$ has all the immoralities that are in $G$.
\end{lemma}

\begin{proof}
We first prove that any triplex in $H$ is in $G$. Assume to the contrary that $H$ has a triplex $(\{A,C\},B)$ that is not in $G$. This is possible if and only if, when line 10 is executed, $H$ has an induced subgraph of one of the following forms:

\begin{table}[H]
\centering
\scalebox{0.75}{
\begin{tabular}{ccccc}
\begin{tikzpicture}[inner sep=1mm]
\node at (0,0) (A) {$A$};
\node at (1,0) (B) {$B$};
\node at (2,0) (C) {$C$};
\path[|-] (A) edge (B);
\path[-|] (B) edge (C);
\end{tikzpicture}
&
\begin{tikzpicture}[inner sep=1mm]
\node at (0,0) (A) {$A$};
\node at (1,0) (B) {$B$};
\node at (2,0) (C) {$C$};
\path[-] (A) edge (B);
\path[-|] (B) edge (C);
\end{tikzpicture}
&
\begin{tikzpicture}[inner sep=1mm]
\node at (0,0) (A) {$A$};
\node at (1,0) (B) {$B$};
\node at (2,0) (C) {$C$};
\path[|-] (A) edge (B);
\path[-] (B) edge (C);
\end{tikzpicture}
&
\begin{tikzpicture}[inner sep=1mm]
\node at (0,0) (A) {$A$};
\node at (1,0) (B) {$B$};
\node at (2,0) (C) {$C$};
\path[|-|] (A) edge (B);
\path[-|] (B) edge (C);
\end{tikzpicture}
&
\begin{tikzpicture}[inner sep=1mm]
\node at (0,0) (A) {$A$};
\node at (1,0) (B) {$B$};
\node at (2,0) (C) {$C$};
\path[|-] (A) edge (B);
\path[|-|] (B) edge (C);
\end{tikzpicture}.
\end{tabular}}
\end{table}

Note that Lemma \ref{lem:adjacencies} implies that $A$ is adjacent to $B$ in $G$, $B$ is adjacent to $C$ in $G$, and that $A$ is not adjacent to $C$ in $G$. This together with the assumption made above that $G$ has no triplex $(\{A,C\},B)$ implies that $B \in S_{AC}$ because, otherwise, the route $A$, $B$, $C$ is $S_{AC}$-open in $G$ contradicting $A \ci_{G} C | S_{AC}$. Now, note that the first, second and fifth induced subgraphs above are impossible because, otherwise, $A \ob B$ would be in $H$ by R2. Likewise, the third and fourth induced subgraphs above are impossible because, otherwise, $B \bo C$ would be in $H$ by R2.

We now prove that any triplex $(\{A,C\},B)$ in $G$ is in $H$. Let the triplex be of the form $A \ra B \la C$. Hence, $B \notin S_{AC}$. Then, when line 10 is executed, $A \bo B \ob C$ is in $H$ by R1, and neither $A \bb B$ nor $B \bb C$ is in $H$ by Lemmas \ref{lem:adjacencies} and \ref{lem:soundness}. Then, the triplex is in $H$. Note that the triplex is an immorality in both $G$ and $H$. Likewise, let the triplex be of the form $A \ra B - C$. Hence, $B \notin S_{AC}$. Then, when line 10 is executed, $A \bo B \ob C$ is in $H$ by R1, and $A \bb B$ is not in $H$ by Lemmas \ref{lem:adjacencies} and \ref{lem:soundness}. Then, the triplex is in $H$. Note that the triplex is a flag in $G$ but it may be an immorality in $H$.
\end{proof}

\begin{lemma}\label{lem:noundirectedcycle}
After line 9, $H$ does not have any induced subgraph of the form
\begin{tabular}{c}
\scalebox{0.75}{
\begin{tikzpicture}[inner sep=1mm]
\node at (0,0) (A) {$A$};
\node at (1,0) (B) {$B$};
\node at (2,0) (C) {$C$};
\path[|-o] (A) edge (B);
\path[-] (B) edge (C);
\path[-] (A) edge [bend left] (C);
\end{tikzpicture}.}
\end{tabular}
\end{lemma}

\begin{proof}
Assume to the contrary that the lemma does not hold. We interpret the execution of line 9 as a sequence of block addings and, for the rest of the proof, one particular sequence of these block addings is fixed. Fixing this sequence is a crucial point upon which some important later steps of the proof are based. Since there may be several induced subgraphs of $H$ of the form under study after line 9, let us consider any of the induced subgraphs
\begin{tabular}{c}
\scalebox{0.75}{
\begin{tikzpicture}[inner sep=1mm]
\node at (0,0) (A) {$A$};
\node at (1,0) (B) {$B$};
\node at (2,0) (C) {$C$};
\path[|-o] (A) edge (B);
\path[-] (B) edge (C);
\path[-] (A) edge [bend left] (C);
\end{tikzpicture}}
\end{tabular}
that appear firstly during execution of line 9 and fix it for the rest of the proof. Now, consider the following cases.

\begin{description}

\item[Case 1] Assume that $A \bo B$ is in $H$ due to R1. Then, after R1 was applied to $A \oo B$, $H$ had an induced subgraph of one of the following forms: 

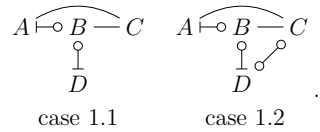
\begin{table}[H]
\centering
\scalebox{0.75}{
\begin{tabular}{cc}
\begin{tikzpicture}[inner sep=1mm]
\node at (0,0) (A) {$A$};
\node at (1,0) (B) {$B$};
\node at (2,0) (C) {$C$};
\node at (1,-1) (D) {$D$};
\path[|-o] (A) edge (B);
\path[-] (B) edge (C);
\path[-] (A) edge [bend left] (C);
\path[|-o] (D) edge (B);
\end{tikzpicture}
&
\begin{tikzpicture}[inner sep=1mm]
\node at (0,0) (A) {$A$};
\node at (1,0) (B) {$B$};
\node at (2,0) (C) {$C$};
\node at (1,-1) (D) {$D$};
\path[|-o] (A) edge (B);
\path[-] (B) edge (C);
\path[-] (A) edge [bend left] (C);
\path[|-o] (D) edge (B);
\path[o-o] (D) edge (C);
\end{tikzpicture}.\\
case 1.1&case 1.2
\end{tabular}}
\end{table}

\begin{description}

\item[Case 1.1] If $B \notin S_{CD}$ then $B \nb C$ is in $H$ by R1, else $B \bn C$ is in $H$ by R2. Either case is a contradiction.

\item[Case 1.2] If $C \notin S_{AD}$ then $A \bn C$ is in $H$ by R1, else $B \nb C$ is in $H$ by R4. Either case is a contradiction.

\end{description}

\item[Case 2] Assume that $A \bo B$ is in $H$ due to R2. Then, after R2 was applied to $A \oo B$, $H$ had an induced subgraph of one of the following forms: 

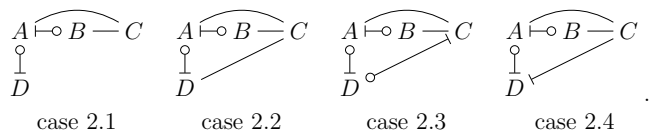
\begin{table}[H]
\centering
\scalebox{0.75}{
\begin{tabular}{cccc}
\begin{tikzpicture}[inner sep=1mm]
\node at (0,0) (A) {$A$};
\node at (1,0) (B) {$B$};
\node at (2,0) (C) {$C$};
\node at (0,-1) (D) {$D$};
\path[|-o] (A) edge (B);
\path[-] (B) edge (C);
\path[-] (A) edge [bend left] (C);
\path[|-o] (D) edge (A);
\end{tikzpicture}
&
\begin{tikzpicture}[inner sep=1mm]
\node at (0,0) (A) {$A$};
\node at (1,0) (B) {$B$};
\node at (2,0) (C) {$C$};
\node at (0,-1) (D) {$D$};
\path[|-o] (A) edge (B);
\path[-] (B) edge (C);
\path[-] (A) edge [bend left] (C);
\path[|-o] (D) edge (A);
\path[-] (D) edge (C);
\end{tikzpicture}
&
\begin{tikzpicture}[inner sep=1mm]
\node at (0,0) (A) {$A$};
\node at (1,0) (B) {$B$};
\node at (2,0) (C) {$C$};
\node at (0,-1) (D) {$D$};
\path[|-o] (A) edge (B);
\path[-] (B) edge (C);
\path[-] (A) edge [bend left] (C);
\path[|-o] (D) edge (A);
\path[o-|] (D) edge (C);
\end{tikzpicture}
&
\begin{tikzpicture}[inner sep=1mm]
\node at (0,0) (A) {$A$};
\node at (1,0) (B) {$B$};
\node at (2,0) (C) {$C$};
\node at (0,-1) (D) {$D$};
\path[|-o] (A) edge (B);
\path[-] (B) edge (C);
\path[-] (A) edge [bend left] (C);
\path[|-o] (D) edge (A);
\path[|-] (D) edge (C);
\end{tikzpicture}.\\
case 2.1&case 2.2&case 2.3&case 2.4
\end{tabular}}
\end{table}

\begin{description}

\item[Case 2.1] If $A \notin S_{CD}$ then $A \nb C$ is in $H$ by R1, else $A \bn C$ is in $H$ by R2. Either case is a contradiction.

\item[Case 2.2] Note that
\begin{tabular}{c}
\scalebox{0.75}{
\begin{tikzpicture}[inner sep=1mm]
\node at (0,0) (A) {$D$};
\node at (1,0) (B) {$A$};
\node at (2,0) (C) {$C$};
\path[|-o] (A) edge (B);
\path[-] (B) edge (C);
\path[-] (A) edge [bend left] (C);
\end{tikzpicture}}
\end{tabular}
cannot be an induced subgraph of $H$ after line 9 because, otherwise, it would contradict the assumption that
\begin{tabular}{c}
\scalebox{0.75}{
\begin{tikzpicture}[inner sep=1mm]
\node at (0,0) (A) {$A$};
\node at (1,0) (B) {$B$};
\node at (2,0) (C) {$C$};
\path[|-o] (A) edge (B);
\path[-] (B) edge (C);
\path[-] (A) edge [bend left] (C);
\end{tikzpicture}}
\end{tabular}
is one of the firstly induced subgraph of that form that appeared during the execution of line 9. Then, $A \bo C$, $A \nb C$, $D \ob C$ or $D \bn C$ must be in $H$ after line 9. However, either of the first two cases is a contradiction. The third case can be reduced to Case 2.3 as follows. The fourth case can be reduced to Case 2.4 similarly. The third case implies that the block at $C$ in $D \ob C$ is added at some moment in the execution of line 9. This moment must happen later than immediately after adding the block at $A$ in $A \bo B$, because immediately after adding this block the situation is the one depicted by the above figure for Case 2.2. Then, when the block at $C$ in $D \ob C$ is added, the situation is the one depicted by the above figure for Case 2.3.

\item[Case 2.3] Assume that the situation of this case occurs at some moment in the execution of line 9. Then, $A \nb C$ is in $H$ after the execution of line 9 by R3, which is a contradiction.

\item[Case 2.4] Assume that the situation of this case occurs at some moment in the execution of line 9. If $C \notin S_{BD}$ then $B \bn C$ is in $H$ after the execution of line 9 by R1, else $B \nb C$ is in $H$ after the execution of line 9 by R2. Either case is a contradiction.

\end{description}

\item[Case 3] Assume that $A \bo B$ is in $H$ due to R3. Then, after R3 was applied to $A \oo B$, $H$ had a subgraph of one of the following forms, where possible additional edges between $C$ and internal nodes of the route $A \bo \ldots \bo D$ are not shown:

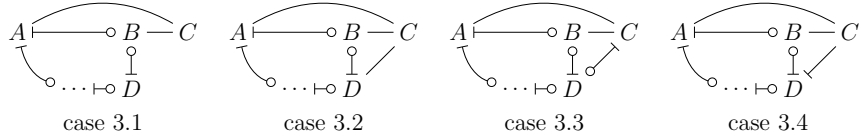
\begin{table}[H]
\centering
\scalebox{0.75}{
\begin{tabular}{cccc}
\begin{tikzpicture}[inner sep=1mm]
\node at (-1,0) (A) {$A$};
\node at (1,0) (B) {$B$};
\node at (2,0) (C) {$C$};
\node at (1,-1) (D) {$D$};
\node at (0,-1) (E) {$\ldots$};
\path[|-o] (A) edge (B);
\path[-] (B) edge (C);
\path[-] (A) edge [bend left] (C);
\path[|-o] (A) edge [bend right] (E);
\path[|-o] (E) edge (D);
\path[|-o] (D) edge (B);
\end{tikzpicture}
&
\begin{tikzpicture}[inner sep=1mm]
\node at (-1,0) (A) {$A$};
\node at (1,0) (B) {$B$};
\node at (2,0) (C) {$C$};
\node at (1,-1) (D) {$D$};
\node at (0,-1) (E) {$\ldots$};
\path[|-o] (A) edge (B);
\path[-] (B) edge (C);
\path[-] (A) edge [bend left] (C);
\path[|-o] (A) edge [bend right] (E);
\path[|-o] (E) edge (D);
\path[|-o] (D) edge (B);
\path[-] (D) edge (C);
\end{tikzpicture}
&
\begin{tikzpicture}[inner sep=1mm]
\node at (-1,0) (A) {$A$};
\node at (1,0) (B) {$B$};
\node at (2,0) (C) {$C$};
\node at (1,-1) (D) {$D$};
\node at (0,-1) (E) {$\ldots$};
\path[|-o] (A) edge (B);
\path[-] (B) edge (C);
\path[-] (A) edge [bend left] (C);
\path[|-o] (A) edge [bend right] (E);
\path[|-o] (E) edge (D);
\path[|-o] (D) edge (B);
\path[o-|] (D) edge (C);
\end{tikzpicture}
&
\begin{tikzpicture}[inner sep=1mm]
\node at (-1,0) (A) {$A$};
\node at (1,0) (B) {$B$};
\node at (2,0) (C) {$C$};
\node at (1,-1) (D) {$D$};
\node at (0,-1) (E) {$\ldots$};
\path[|-o] (A) edge (B);
\path[-] (B) edge (C);
\path[-] (A) edge [bend left] (C);
\path[|-o] (A) edge [bend right] (E);
\path[|-o] (E) edge (D);
\path[|-o] (D) edge (B);
\path[|-] (D) edge (C);
\end{tikzpicture}.\\
case 3.1&case 3.2&case 3.3&case 3.4
\end{tabular}}
\end{table}

Note that $C$ cannot belong to the route $A \bo \ldots \bo D$ because, otherwise, R3 could not have been applied since the cycle $A \bo \ldots \bo D \bo B \no A$ would not have been chordless.

\begin{description}

\item[Case 3.1] If $B \notin S_{CD}$ then $B \nb C$ is in $H$ by R1, else $B \bn C$ is in $H$ by R2. Either case is a contradiction.

\item[Case 3.2] Note that
\begin{tabular}{c}
\scalebox{0.75}{
\begin{tikzpicture}[inner sep=1mm]
\node at (0,0) (A) {$D$};
\node at (1,0) (B) {$B$};
\node at (2,0) (C) {$C$};
\path[|-o] (A) edge (B);
\path[-] (B) edge (C);
\path[-] (A) edge [bend left] (C);
\end{tikzpicture}}
\end{tabular}
cannot be an induced subgraph of $H$ after line 9 because, otherwise, it would contradict the assumption that
\begin{tabular}{c}
\scalebox{0.75}{
\begin{tikzpicture}[inner sep=1mm]
\node at (0,0) (A) {$A$};
\node at (1,0) (B) {$B$};
\node at (2,0) (C) {$C$};
\path[|-o] (A) edge (B);
\path[-] (B) edge (C);
\path[-] (A) edge [bend left] (C);
\end{tikzpicture}}
\end{tabular}
is one of the firstly induced subgraph of that form that appeared during the execution of line 9. Then, $B \bo C$, $B \nb C$, $D \ob C$ or $D \bn C$ must be in $H$ after line 9. However, either of the first two cases is a contradiction. The third case can be reduced to Case 3.3 as follows. The fourth case can be reduced to Case 3.4 similarly. The third case implies that the block at $C$ in $D \ob C$ is added at some moment in the execution of line 9. This moment must happen later than immediately after adding the block at $A$ in $A \bo B$, because immediately after adding this block the situation is the one depicted by the above figure for Case 3.2. Then, when the block at $C$ in $D \ob C$ is added, the situation is the one depicted by the above figure for Case 3.3.

\item[Case 3.3] Assume that the situation of this case occurs at some moment in the execution of line 9. Then, $B \nb C$ is in $H$ after the execution of line 9 by R3, which is a contradiction.

\item[Case 3.4] Assume that the situation of this case occurs at some moment in the execution of line 9. Note that $C$ cannot be adjacent to any node of the route $A \bo \ldots \bo D$ besides $A$ and $D$. To see it, assume to the contrary that $C$ is adjacent to some nodes $E_1, \ldots, E_n \neq A, D$ of the route $A \bo \ldots \bo D$. Assume without loss of generality that $E_i$ is closer to $A$ in the route than $E_{i+1}$ for all $1 \leq i < n$. Now, note that $E_n \bo C$ must be in $H$ after the execution of line 9 by R3. This implies that $E_{n-1} \bo C$ must be in $H$ after the execution of line 9 by R3. By repeated application of this argument, we can conclude that $E_1 \bo C$ must be in $H$ after the execution of line 9 and, thus, $A \bn C$ must be in $H$ after the execution of line 9 by R3, which is a contradiction.

\end{description}

\item[Case 4] Assume that $A \bo B$ is in $H$ due to R4. Then, after R4 was applied to $A \oo B$, $H$ had an induced subgraph of one of the following forms:

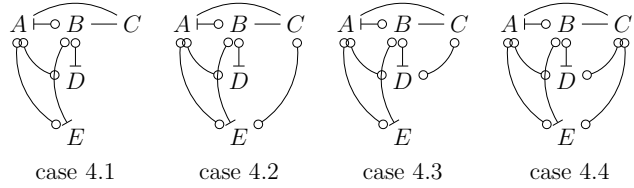
\begin{table}[H]
\centering
\scalebox{0.75}{
\begin{tabular}{cccc}
\begin{tikzpicture}[inner sep=1mm][inner sep=1mm]
\node at (0,0) (A) {$A$};
\node at (1,0) (B) {$B$};
\node at (2,0) (C) {$C$};
\node at (1,-1) (D) {$D$};
\node at (1,-2) (E) {$E$};
\path[|-o] (A) edge (B);
\path[-] (B) edge (C);
\path[-] (A) edge [bend left] (C);
\path[o-o] (A) edge [bend right] (D);
\path[|-o] (D) edge (B);
\path[o-o] (A) edge [bend right] (E);
\path[|-o] (E) edge [bend left] (B);
\end{tikzpicture}
&
\begin{tikzpicture}[inner sep=1mm][inner sep=1mm]
\node at (0,0) (A) {$A$};
\node at (1,0) (B) {$B$};
\node at (2,0) (C) {$C$};
\node at (1,-1) (D) {$D$};
\node at (1,-2) (E) {$E$};
\path[|-o] (A) edge (B);
\path[-] (B) edge (C);
\path[-] (A) edge [bend left] (C);
\path[o-o] (A) edge [bend right] (D);
\path[|-o] (D) edge (B);
\path[o-o] (A) edge [bend right] (E);
\path[|-o] (E) edge [bend left] (B);
\path[o-o] (E) edge [bend right] (C);
\end{tikzpicture}
&
\begin{tikzpicture}[inner sep=1mm][inner sep=1mm]
\node at (0,0) (A) {$A$};
\node at (1,0) (B) {$B$};
\node at (2,0) (C) {$C$};
\node at (1,-1) (D) {$D$};
\node at (1,-2) (E) {$E$};
\path[|-o] (A) edge (B);
\path[-] (B) edge (C);
\path[-] (A) edge [bend left] (C);
\path[o-o] (A) edge [bend right] (D);
\path[|-o] (D) edge (B);
\path[o-o] (A) edge [bend right] (E);
\path[|-o] (E) edge [bend left] (B);
\path[o-o] (D) edge [bend right] (C);
\end{tikzpicture}
&
\begin{tikzpicture}[inner sep=1mm]
\node at (0,0) (A) {$A$};
\node at (1,0) (B) {$B$};
\node at (2,0) (C) {$C$};
\node at (1,-1) (D) {$D$};
\node at (1,-2) (E) {$E$};
\path[|-o] (A) edge (B);
\path[-] (B) edge (C);
\path[-] (A) edge [bend left] (C);
\path[o-o] (A) edge [bend right] (D);
\path[|-o] (D) edge (B);
\path[o-o] (A) edge [bend right] (E);
\path[|-o] (E) edge [bend left] (B);
\path[o-o] (D) edge [bend right] (C);
\path[o-o] (E) edge [bend right] (C);
\end{tikzpicture}.\\
case 4.1&case 4.2&case 4.3&case 4.4
\end{tabular}}
\end{table}

\begin{description}

\item[Cases 4.1-4.3] If $B \notin S_{CD}$ or $B \notin S_{CE}$ then $B \nb C$ is in $H$ by R1, else $B \bn C$ is in $H$ by R2. Either case is a contradiction.

\item[Case 4.4] Assume that $C \in S_{DE}$. Then, $B \nb C$ is in $H$ by R4, which is a contradiction. On the other hand, assume that $C \notin S_{DE}$. Then, it follows from applying R1 that $H$ has an induced subgraph of the form

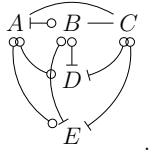
\begin{table}[H]
\centering
\scalebox{0.75}{
\begin{tikzpicture}[inner sep=1mm]
\node at (0,0) (A) {$A$};
\node at (1,0) (B) {$B$};
\node at (2,0) (C) {$C$};
\node at (1,-1) (D) {$D$};
\node at (1,-2) (E) {$E$};
\path[|-o] (A) edge (B);
\path[-] (B) edge (C);
\path[-] (A) edge [bend left] (C);
\path[o-o] (A) edge [bend right] (D);
\path[|-o] (D) edge (B);
\path[o-o] (A) edge [bend right] (E);
\path[|-o] (E) edge [bend left] (B);
\path[|-o] (D) edge [bend right] (C);
\path[|-o] (E) edge [bend right] (C);
\end{tikzpicture}.}
\end{table}

Note that $A \in S_{DE}$ because, otherwise, R4 would not have been applied. Then, $A \bn C$ is in $H$ by R4, which is a contradiction.

\end{description}

\end{description}

\end{proof}

\begin{lemma}\label{lem:oppositeblock}
After line 9, every chordless cycle $\rho: V_1, \ldots, V_n=V_1$ in $H$ that has an edge $V_i \bn V_{i+1}$ also has an edge $V_j \nb V_{j+1}$.
\end{lemma}

\begin{proof}
Assume for a contradiction that $\rho$ is of the length three such that $V_{1} \bn V_{2}$ occur and neither $V_{2} \nb V_{3}$ nor $V_{1} \bn V_{3}$ occur. Note that $V_{2} \bb V_{3}$ cannot occur either because, otherwise, $V_{1} \bn V_{3}$ or $V_{1} \bb V_{3}$ must occur by R3. Since the former contradicts the assumption, then the latter must occur. However, this implies that $V_{1} \bb V_{2}$ must occur by R3, which contradicts the assumption. Similarly, $V_{1} \bb V_{3}$ cannot occur either. Then, $\rho$ is of one of the following forms:

\begin{table}[H]
\centering
\scalebox{0.75}{
\begin{tabular}{ccc}
\begin{tikzpicture}[inner sep=1mm]
\node at (0,0) (A) {$V_1$};
\node at (1,0) (B) {$V_{2}$};
\node at (2,0) (C) {$V_{3}$};
\path[|-] (A) edge (B);
\path[-] (B) edge (C);
\path[-] (A) edge [bend left] (C);
\end{tikzpicture}
&
\begin{tikzpicture}[inner sep=1mm]
\node at (0,0) (A) {$V_1$};
\node at (1,0) (B) {$V_{2}$};
\node at (2,0) (C) {$V_{3}$};
\path[|-] (A) edge (B);
\path[o-] (B) edge (C);
\path[-|] (A) edge [bend left] (C);
\end{tikzpicture}
&
\begin{tikzpicture}[inner sep=1mm]
\node at (0,0) (A) {$V_1$};
\node at (1,0) (B) {$V_{2}$};
\node at (2,0) (C) {$V_{3}$};
\path[|-] (A) edge (B);
\path[|-] (B) edge (C);
\path[-] (A) edge [bend left] (C);
\end{tikzpicture}.
\end{tabular}}
\end{table}

The first form is impossible by Lemma \ref{lem:noundirectedcycle}. The second form is impossible because, otherwise, $V_{2} \ob V_{3}$ would occur by R3. The third form is impossible because, otherwise, $V_{1} \bn V_{3}$ would be occur by R3. Thus, the lemma holds for cycles of length three.

Assume for a contradiction that $\rho$ is of length greater than three and has an edge $V_{i} \bn V_{i+1}$ but no edge $V_{j} \nb V_{j+1}$. Note that if $V_l \bo V_{l+1} \oo V_{l+2}$ is a subroute of $\rho$, then either $V_{l+1} \bo V_{l+2}$ or $V_{l+1} \nb V_{l+2}$ is in $\rho$ by R1 and R2. Since $\rho$ has no edge $V_j \nb V_{j+1}$, $V_{l+1} \bo V_{l+2}$ is in $\rho$. By repeated application of this reasoning together with the fact that $\rho$ has an edge $V_i \bn V_{i+1}$, we can conclude that every edge in $\rho$ is $V_k \bo V_{k+1}$. Then, by repeated application of R3, observe that every edge in $\rho$ is $V_{k} \bb V_{k+1}$, which contradicts the assumption.
\end{proof}

\begin{theorem}\label{the:acyclic}
After line 10, $H$ is triplex equivalent to $G$ and it has no semidirected cycle.
\end{theorem}

\begin{proof}
Lemma \ref{lem:adjacencies} implies that $G$ and $H$ have the same adjacencies. Lemma \ref{lem:triplexes} implies that $G$ and $H$ have the same triplexes. Lemma \ref{lem:oppositeblock} implies that $H$ has no semidirected chordless cycle, which implies that $H$ has no semidirected cycle. To see the latter implication, assume to the contrary that $H$ has no semidirected chordless cycle but that it has a semidirected cycle $\rho: V_1, \ldots, V_n=V_1$ with a chord between $V_i$ and $V_j$ with $i < j$. Then, divide $\rho$ into the cycles $\rho_L: V_1, \ldots, V_i, V_j, \ldots, V_n=V_1$ and $\rho_R: V_i, \ldots, V_j, V_i$. Note that $\rho_L$ or $\rho_R$ is a semidirected cycle. Then, $H$ has a semidirected cycle that is shorter than $\rho$. By repeated application of this reasoning, we can conclude that $H$ has a semidirected chordless cycle, which is a contradiction.
\end{proof}

\subsection{Discussion}\label{sec:discussion}

In this section, we have presented an algorithm for learning an AMP CG a given probability distribution $p$ is faithful to. In practice, of course, we do not usually have access to $p$ but to a finite sample from it. Our algorithm can easily be modified to deal with this situation: Replace $A \ci_p B | S$ in line 5 with a hypothesis test, preferably with one that is consistent so that the resulting algorithm is asymptotically correct.

It is worth mentioning that, whereas R1, R2 and R4 only involve three or four nodes, R3 may involve many more. Hence, it would be desirable to replace R3 with a simpler rule such as 

\begin{table}[H]
\centering
\scalebox{0.75}{
\begin{tabular}{ccc}
\begin{tabular}{c}
\begin{tikzpicture}[inner sep=1mm]
\node at (0,0) (A) {$A$};
\node at (1,0) (B) {$B$};
\node at (2,0) (C) {$C$};
\path[|-o] (A) edge (B);
\path[|-o] (B) edge (C);
\path[o-o] (A) edge [bend left] (C);
\end{tikzpicture}
\end{tabular}
& $\Rightarrow$ &
\begin{tabular}{c}
\begin{tikzpicture}[inner sep=1mm]
\node at (0,0) (A) {$A$};
\node at (1,0) (B) {$B$};
\node at (2,0) (C) {$C$};
\path[|-o] (A) edge (B);
\path[|-o] (B) edge (C);
\path[|-o] (A) edge [bend left] (C);
\end{tikzpicture}.
\end{tabular}
\end{tabular}}
\end{table}

Unfortunately, we have not succeeded so far in proving the correctness of our algorithm with such a simpler rule. Note that the output of our algorithm will be the same whether we keep R3 or we replace it with a simpler sound rule. The only benefit of the simpler rule may be a decrease in running time.

We have shown in Lemma \ref{lem:triplexes} that, after line 10, $H$ has all the immoralities in $G$ or, in other words, every flag in $H$ is in $G$. The following lemma strengthens this fact.

\begin{lemma}\label{lem:essentialline}
After line 10, every flag in $H$ is in every CG $F$ that is triplex equivalent to $G$.
\end{lemma}

\begin{proof}
Note that every flag $A \ra B - C$ in $H$ after line 10 is due to an induced subgraph of $H$ of the form $A \bn B \bb C$ after line 9 because $A \bn B - C$ is excluded by R1 and R2. Note also that all the blocks in $H$ follow from the adjacencies and triplexes in $G$ by repeated application of R1-R4. Since $G$ and $F$ have the same adjacencies and triplexes, all the blocks in $H$ hold in both $G$ and $F$ by Lemma \ref{lem:soundness}. 
\end{proof}

A CG whose every flag is in every other triplex equivalent CG is called a deflagged graph by \citet[Proposition 8]{RoveratoandStudeny2006}. Therefore, the lemma above implies that our algorithm outputs a deflagged graph. Note that there may be several deflagged graphs that are triplex equivalent to $G$. Unfortunately, not every directed edge in the output of our algorithm is in every deflagged graph that is triplex equivalent to $G$, as the following example illustrates (note that both $G$ and $H$ are deflagged graphs). 

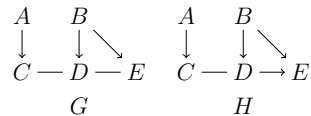
\begin{table}[H]
\centering
\scalebox{0.75}{
\begin{tabular}{cc}
\begin{tikzpicture}[inner sep=1mm]
\node at (0,0) (A) {$A$};
\node at (1,0) (B) {$B$};
\node at (0,-1) (C) {$C$};
\node at (1,-1) (D) {$D$};
\node at (2,-1) (E) {$E$};
\path[->] (A) edge (C);
\path[->] (B) edge (D);
\path[-] (C) edge (D);
\path[-] (D) edge (E);
\path[->] (B) edge (E);
\end{tikzpicture}
&
\begin{tikzpicture}[inner sep=1mm]
\node at (0,0) (A) {$A$};
\node at (1,0) (B) {$B$};
\node at (0,-1) (C) {$C$};
\node at (1,-1) (D) {$D$};
\node at (2,-1) (E) {$E$};
\path[->] (A) edge (C);
\path[->] (B) edge (D);
\path[-] (C) edge (D);
\path[->] (D) edge (E);
\path[->] (B) edge (E);
\end{tikzpicture}\\
$G$&$H$
\end{tabular}}
\end{table}

Therefore, our algorithm outputs a deflagged graph but not what \cite{RoveratoandStudeny2006} call the largest deflagged graph. The latter is a distinguished member of a class of triplex equivalent CGs. Fortunately, the largest deflagged graph can easily be obtained from any deflagged graph in the class \cite[Corollary 17]{RoveratoandStudeny2006}. 

Another distinguished member of a class of triplex equivalent CGs is the so-called essential graph $G^*$ \citep{AnderssonandPerlman2006}: An edge $A \ra B$ is in $G^*$ if and only if $A \la B$ is in no member of the class. Unfortunately, our algorithm does not output an essential graph either, as the following example illustrates.

\begin{table}[H]
\centering
\scalebox{0.75}{
\begin{tabular}{cc}
\begin{tikzpicture}[inner sep=1mm]
\node at (0,0) (A) {$A$};
\node at (1,0) (B) {$B$};
\node at (0,-1) (C) {$C$};
\node at (1,-1) (D) {$D$};
\node at (2,-1) (E) {$E$};
\path[-] (A) edge (C);
\path[-] (B) edge (D);
\path[-] (C) edge (D);
\path[-] (D) edge (E);
\path[-] (A) edge (B);
\end{tikzpicture}
&
\begin{tikzpicture}[inner sep=1mm]
\node at (0,0) (A) {$A$};
\node at (1,0) (B) {$B$};
\node at (0,-1) (C) {$C$};
\node at (1,-1) (D) {$D$};
\node at (2,-1) (E) {$E$};
\path[-] (A) edge (C);
\path[-] (B) edge (D);
\path[-] (C) edge (D);
\path[->] (D) edge (E);
\path[-] (A) edge (B);
\end{tikzpicture}\\
$G=H$&$G^*$
\end{tabular}}
\end{table}

It is worth mentioning that a characterization of essential graphs that is more efficient than the one introduced above is available \citep[Theorem 5.1]{AnderssonandPerlman2006}. Also, an efficient algorithm for constructing the essential graph from any member of the class has been proposed \citep[Section 7]{AnderssonandPerlman2004}. As far as we know, the correctness of the algorithm has not been proven though.

The correctness of our algorithm lies upon the assumption that $p$ is faithful to some CG. This is a strong requirement that we would like to weaken, e.g. by replacing it with the milder assumption that $p$ satisfies the composition property. Correct algorithms for learning directed and acyclic graphs (a.k.a. Bayesian networks) under the composition property assumption exist \citep{ChickeringandMeek2002,Nielsenetal.2003}. We have recently developed a correct algorithm for learning LWF CGs under the composition property \citep{Pennaetal.2012}. The way in which these algorithms proceed (a.k.a. score+search based approach) is rather different from that of the algorithm presented in this section (a.k.a. constraint based approach). In a nutshell, they can be seen as consisting of two phases: A first phase that starts from the empty graph $H$ and adds single edges to it until $p$ is Markovian with respect to $H$, and a second phase that removes single edges from $H$ until $p$ is Markovian with respect to $H$ and $p$ is not Markovian with respect to any CG $F$ such that $I(H) \subseteq I(F)$. The success of the first phase is guaranteed by the composition property assumption, whereas the success of the second phase is guaranteed by the so-called Meek's conjecture \citep{Meek1997}. Specifically, given two directed and acyclic graphs $F$ and $H$ such that $I(H) \subseteq I(F)$, Meek's conjecture states that we can transform $F$ into $H$ by a sequence of operations such that, after each operation, $F$ is a directed and acyclic graph and $I(H) \subseteq I(F)$. The operations consist in adding a single edge to $F$, or replacing $F$ with a triplex equivalent directed and acyclic graph. Meek's conjecture was proven to be true in \citep[Theorem 4]{Chickering2002}. The extension of Meek's conjecture to LWF CGs was proven to be true in \citep[Theorem 1]{Penna2011}. Unfortunately, the extension of Meek's conjecture to AMP CGs does not hold, as the following example illustrates.

\begin{example}
Consider the AMP CGs $F$ and $H$ below.

\begin{table}[H]
\centering
\scalebox{0.75}{
\begin{tabular}{cc}
\begin{tikzpicture}[inner sep=1mm]
\node at (0,0) (A) {$A$};
\node at (1,0) (B) {$B$};
\node at (-1,-1) (C) {$C$};
\node at (0,-1) (D) {$D$};
\node at (1,-1) (E) {$E$};
\path[->] (A) edge (D);
\path[->] (B) edge (E);
\path[-] (C) edge (D);
\path[-] (D) edge (E);
\end{tikzpicture}
&
\begin{tikzpicture}[inner sep=1mm]
\node at (0,0) (A) {$A$};
\node at (1,0) (B) {$B$};
\node at (-1,-1) (C) {$C$};
\node at (0,-1) (D) {$D$};
\node at (1,-1) (E) {$E$};
\path[->] (A) edge (D);
\path[-] (B) edge (E);
\path[-] (C) edge (D);
\path[-] (D) edge (E);
\path[-] (B) edge (D);
\end{tikzpicture}\\
$F$&$H$
\end{tabular}}
\end{table}

We can describe $I(F)$ and $I(H)$ by listing all the separators between any pair of distinct nodes. We indicate whether the separators correspond to $F$ or $H$ with a superscript. Specifically, 

\begin{itemize}
\item ${\mathcal S}^F_{AD}={\mathcal S}^F_{BE}={\mathcal S}^F_{CD}={\mathcal S}^F_{DE}=\emptyset$,

\item ${\mathcal S}^F_{AB} = \{ \emptyset,\{C\},\{D\},\{E\},\{C,D\},\{C,E\} \}$,

\item ${\mathcal S}^F_{AC} = \{ \emptyset,\{B\},\{E\},\{B,E\} \}$,

\item ${\mathcal S}^F_{AE} = \{ \emptyset,\{B\},\{C\},\{B,C\} \}$,

\item ${\mathcal S}^F_{BC} = \{ \emptyset,\{A\},\{D\},\{A,D\},\{A,D,E\} \}$,

\item ${\mathcal S}^F_{BD} = \{ \emptyset,\{A\},\{C\},\{A,C\} \}$, and

\item ${\mathcal S}^F_{CE} = \{ \{A,D\},\{A,B,D\} \}$.
\end{itemize}

Likewise, 

\begin{itemize}
\item ${\mathcal S}^H_{AD}={\mathcal S}^H_{BD}={\mathcal S}^H_{BE}={\mathcal S}^H_{CD}={\mathcal S}^H_{DE}=\emptyset$,

\item ${\mathcal S}^H_{AB} = \{ \emptyset,\{C\},\{E\},\{C,E\} \}$,

\item ${\mathcal S}^H_{AC} = \{ \emptyset,\{B\},\{E\},\{B,E\} \}$,

\item ${\mathcal S}^H_{AE} = \{ \emptyset,\{B\},\{C\},\{B,C\} \}$,

\item ${\mathcal S}^H_{BC} = \{ \{A,D\},\{A,D,E\} \}$, and

\item ${\mathcal S}^H_{CE} = \{ \{A,D\},\{A,B,D\} \}$.
\end{itemize}

Then, $I(H) \subseteq I(F)$ because ${\mathcal S}^H_{XY} \subseteq {\mathcal S}^F_{XY}$ for all $X, Y \in \{A, B, C, D, E\}$ with $X \neq Y$. However, there is no CG that is triplex equivalent to $F$ or $H$ and, obviously, one cannot transform $F$ into $H$ by adding a single edge.
\end{example}

While the example above compromises the development of score+search learning algorithms that are correct and efficient under the composition property assumption, it is not clear to us whether it also does it for constraint based algorithms. This is something we plan to study.

\section{Maximal Covariance-Concentration Graphs}\label{sec:mccg}

As mentioned in the introduction, AMP CGs are not closed under marginalization, which leads us to the problem of how to represent the result of marginalizing out some nodes in an AMP CG. In this section, we present the partial solution to this problem that we have obtained so far. Specifically, we introduce and study a new family of graphical models that we call maximal covariance-concentration graphs. These new models solve the problem at hand partially, because each of them represents the result of marginalizing out some nodes in some AMP CG. Unfortunately, our new models do not solve the problem completely, because they do not represent the result of marginalizing out any nodes in any AMP CG.

First, we define covariance-concentration graphs (CCGs) as graphs whose every edge is undirected or bidirected. A node $B$ in a path $\rho$ in a CCG is called a triplex node in $\rho$ if $A \aa B \aa C$, $A \aa B - C$ or $A - B \aa C$ is a subpath of $\rho$. Let $X$, $Y$ and $Z$ denote three pairwise disjoint subsets of $V$. A path $\rho$ in a CCG $G$ is said to be $Z$-open when 
\begin{itemize}
\item every triplex node in $\rho$ is in $Z$, and
\item every non-triplex node in $\rho$ is not in $Z$ or has some spouse in $G$.
\end{itemize}

When there is no path in $G$ between a node in $X$ and a node in $Y$ that is $Z$-open, we say that $X$ is separated from $Y$ given $Z$ and denote it as $X \ci_G Y | Z$. We denote by $X \nci_G Y | Z$ that $X \ci_G Y | Z$ does not hold. The independence model induced by $G$ is the set of separations $X \ci_G$ $Y | Z$.

Typically, every missing edge in a graphical model corresponds to a separation. However, this is not true for CCGs. For instance, the CCG $G$ below does not contain any edge between $B$ and $D$ but $B \nci_G D | Z$ for all $Z \subseteq V \setminus \{B, D\}$. Likewise, $G$ does not contain any edge between $A$ and $E$ but $A \nci_G E | Z$ for all $Z \subseteq V \setminus \{A, E\}$.

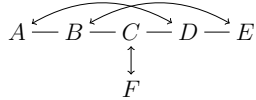
\begin{table}[H]
\centering
\scalebox{0.75}{
\begin{tabular}{c}
\begin{tikzpicture}[inner sep=1mm]
\node at (0,0) (A) {$A$};
\node at (1,0) (B) {$B$};
\node at (2,0) (C) {$C$};
\node at (3,0) (D) {$D$};
\node at (4,0) (E) {$E$};
\node at (2,-1) (F) {$F$};
\path[-] (A) edge (B);
\path[-] (B) edge (C);
\path[-] (C) edge (D);
\path[-] (D) edge (E);
\path[<->] (A) edge [bend left] (D);
\path[<->] (B) edge [bend left] (E);
\path[<->] (C) edge (F);
\end{tikzpicture}
\end{tabular}}
\end{table}

In order to avoid the problem above, we focus in this paper on what we call maximal CCGs (MCCGs), which are those CCGs that have
\begin{itemize}
\item no induced subgraph $A - C - B$ such that $C$ has some spouse, and
\item no cycle $A - \ldots - B \aa A$.
\end{itemize}

Hereinafter, we refer to the two constrains on CCGs above as C1 and C2, respectively. As Theorem \ref{the:local1} shows, every missing edge in a MCCG corresponds to a separation. So, no edge can be added to a MCCG without changing the independence model induced by it, hence the name. Note that a MCCG $G$ represents the same separations over $V$ as the AMP CG $H$ obtained by replacing every bidirected edge $A \aa B$ in $G$ with $A \la H_{AB} \ra B$. Therefore, $G$ represents the marginal independence model of $H$ over $V$. See Section \ref{sec:discussion2} for a discussion on the relationship of MCCGs with other families of graphical models. Note also that both covariance and concentration graphs are MCCGs, and the definitions of separation for covariance and concentration graphs are special cases of the one introduced above for MCCGs (recall Section \ref{sec:introduction}). Therefore, MCCGs unify and generalize covariance and concentration graphs.

Note that if a MCCG has a subgraph $A - C - B$ such that $C$ has some spouse, then the constraint C1 implies that there must be an edge between $A$ and $B$ in the MCCG, whereas the constraint C2 implies that the edge must be undirected. Therefore, if a MCCG has a path $A=V_1 - V_2 - \ldots - V_n=B$ such that $V_i$ has some spouse for all $1 < i < n$, then the edge $V_1 - V_n$ must be in the MCCG. Therefore, the independence model induced by a MCCG is the same whether we use the definition of $Z$-open path above or the following simpler one. A path $\rho$ in a MCCG is said to be $Z$-open when 
\begin{itemize}
\item every triplex node in $\rho$ is in $Z$, and
\item every non-triplex node in $\rho$ is not in $Z$.
\end{itemize}

The theorem below shows that the independence models induced by MCCGs are not arbitrary in the probabilistic framework.

\begin{theorem}\label{the:faithfulness}
For any MCCG $G$, there exists a regular Gaussian probability distribution $p$ that is faithful to $G$.
\end{theorem}

\begin{proof}
It suffices to replace every bidirected edge $A \aa B$ in $G$ with $A \la H_{AB} \ra B$ to create an AMP CG $H$, apply Theorem 6.1 by \cite{Levitzetal.2001} to conclude that there exists a regular Gaussian probability distribution $q$ that is faithful to $H$, and then let $p$ be the marginal probability distribution of $q$ over $V$.
\end{proof}

\begin{corollary}\label{cor:wtc}
The independence models induced by MCCGs satisfy the graphoid, composition and weak transitivity properties.
\end{corollary}

\begin{proof}
It follows from Theorem \ref{the:faithfulness} by just noting that the set of independencies in any regular Gaussian probability distribution satisfy the properties mentioned \cite[Sections 2.2.2, 2.3.5 and 2.3.6]{Studeny2005}.
\end{proof}

Another interesting property of MCCGs is that they are closed under marginalization: For every MCCG $G$ and $U \subseteq V$, there exists a so-called marginal MCCG $G^U$ over $U$ such that $X \ci_{G^U} Y | Z$ if and only if $X \ci_G Y | Z$ for all $X$, $Y$ and $Z$ pairwise disjoint subsets of $U$. Specifically, $G^U$ can be obtained from $G_U$ by adding an edge $A - B$ to it if $G$ has a path $A - \ldots - B$ such that $A$ and $B$ are the only nodes in the path that are in $U$.

Finally, we show below that the independence model induced by a MCCG coincides with certain closure of certain separations. We define the local separation base of a MCCG $G$ as the following set of separations:
\begin{itemize}
\item $A \ci B$ for all $A, B \in V$ such that $A$ and $B$ are not adjacent in $G$ and are in different undirected connectivity components of $G$, and
\item $A \ci B | ne_G(A)$ for all $A, B \in V$ such that $A$ and $B$ are not adjacent in $G$ and are in the same undirected connectivity component of $G$.
\end{itemize}

We define the closure of the local separation base of $G$, denoted as $cl(G)$, as the set of separations that are in the base plus those that can be derived from it by applying the graphoid, composition and weak transitivity properties. We denote the separations in $cl(G)$ as $X \ci_{cl(G)} Y | Z$.

\begin{theorem}\label{the:local1}
For any MCCG $G$, if $X \ci_{cl(G)} Y | Z$ then $X \ci_G Y | Z$.
\end{theorem}

\begin{proof}
Since the independence model induced by $G$ satisfies the graphoid, composition and weak transitivity properties by Corollary \ref{cor:wtc}, it suffices to prove that the local separation base of $G$ is a subset of the independence model induced by $G$. We prove this next. If two non-adjacent nodes $A$ and $B$ are not in the same undirected connectivity component of $G$, then every path between $A$ and $B$ in $G$ has some triplex node. Therefore, $A \ci_G B$. On the other hand, if $A$ and $B$ are in the same undirected connectivity component of $G$, then every path between $A$ and $B$ in $G$ falls within one of the following cases.

\begin{description}
\item[Case 1] $A=V_1 - V_2 - V_3 \ldots V_n=B$ such that $V_2$ has no spouse in $G$. Then, this path is not $ne_G(V_1)$-open.

\item[Case 2] $A=V_1 - V_2 - \ldots - V_m - V_{m+1} - V_{m+2} \ldots V_n=B$ such that $V_i$ has some spouse in $G$ for all $2 \leq i \leq m$ and $V_{m+1}$ has no spouse in $G$. Note that $V_i \in ne_G(V_1)$ by constraints C1 and C2 for all $2 \leq i \leq m+1$. Then, this path is not $ne_G(V_1)$-open.

\item[Case 3] $A=V_1 - V_2 \aa V_3 \ldots V_n=B$. Note that $V_3 \neq V_n$ and $V_3 \notin ne_G(V_1)$ by constraint C2. Then, $V_3$ is a triplex node in this path and, thus, this path is not $ne_G(V_1)$-open.

\item[Case 4] $A=V_1 - V_2 - \ldots - V_m - V_{m+1} \aa V_{m+2} \ldots V_n=B$ such that $V_i$ has some spouse in $G$ for all $2 \leq i \leq m$. Note that $V_{m+2} \neq V_n$ and $V_{m+2} \notin ne_G(V_1)$ by constraint C2. Then, $V_{m+2}$ is a triplex node in this path and, thus, this path is not $ne_G(V_1)$-open.

\item[Case 5] $A=V_1 \aa V_2 \ldots V_n=B$. Note that $V_2 \neq V_n$ by constraint C2. Then, $V_2$ is a triplex node in this path and, thus, this path is not $ne_G(V_1)$-open.

\end{description}

Consequently, $A \ci_G B | ne_G(A)$.

\end{proof}

\begin{lemma}\label{lem:aux}
Let $G$ be a MCCG, $A, B \in V$ and $Z \subseteq V \setminus \{A, B\}$. If $A \ci_G B | Z$ and a node $C \in Z$ has some spouse in $G$, then $A \ci_G B | Z \setminus C$.
\end{lemma}

\begin{proof}
Assume the contrary. Then, there is a path $\rho$ between $A$ and $B$ in $G$ that is $(Z \setminus C)$-open. Moreover, $C$ must occur in $\rho$ because, otherwise, $\rho$ would also be $Z$-open which would contradict the assumption that $A \ci_G B | Z$. For the same reason, $C$ must be a non-triplex node in $\rho$. Let $D - C - E$ be a subpath of $\rho$. Note that the edge $D - E$ is in $G$ by definition of MCCG, because $C$ has some spouse in $G$. Then, the path obtained from $\rho$ by replacing the subpath $D - C - E$ with the edge $D - E$ is $Z$-open. However, this contradicts the assumption that $A \ci_G B | Z$.
\end{proof}

\begin{theorem}\label{the:local2}
For any MCCG $G$, if $X \ci_G Y | Z$ then $X \ci_{cl(G)} Y | Z$.
\end{theorem}

\begin{proof}
Since the independence model induced by $G$ satisfies the decomposition property and $cl(G)$ satisfies the composition property, it suffices to prove that if $A \ci_G B | Z$ then $A \ci_{cl(G)} B | Z$ with $A, B \in V$ and $Z \subseteq V \setminus \{A, B\}$. We prove this result by induction on $|Z|$. If $|Z|=0$, then $A$ and $B$ must be in different undirected connectivity components of $G$. Consequently, $A \ci_{cl(G)} B$. Assume as induction hypothesis that the theorem holds for $|Z|<l$. We now prove it for $|Z|=l$. Consider the following cases.

\begin{description}
\item[Case 1] $A$ and $B$ are in the same undirected connectivity component $K$ of $G$.

\begin{description}
\item[Case 1.1] All the nodes in $Z$ are in $K$. Then, $A \ci_{cl(G)} B | Z$ \cite[Theorem 3.7]{Lauritzen1996}.

\item[Case 1.2] There is some node $C \in Z$ that is not in $K$ such that $C$ is a spouse of some node in $K$ and $A \nci_G C | Z \setminus C$. Then, $B \ci_G C | Z \setminus C$. To see it, assume the contrary. Then, $A \nci_G C | Z \setminus C$ and $B \nci_G C | Z \setminus C$ imply $A \nci_G B | Z \setminus C$ or $A \nci_G B | Z$ by weak transitivity, which implies $A \nci_G B | Z \setminus C$ because $A \ci_G B | Z$ by assumption. However, this contradicts Lemma \ref{lem:aux}.

Finally, note that $B \ci_G C | Z \setminus C$ implies $B \ci_{cl(G)} C | Z \setminus C$ by the induction hypothesis. Note also that $A \ci_G B | Z \setminus C$ by Lemma \ref{lem:aux} and, thus, $A \ci_{cl(G)} B | Z \setminus C$ by the induction hypothesis. Then, $A \ci_{cl(G)} B | Z$ by symmetry, composition and weak union.

\item[Case 1.3] Cases 1.1 and 1.2 do not apply. Let $C$ be any node in $Z$ that is not in $K$. Then, $A \ci_G C | Z \setminus C$. Note also that $A \ci_G B | Z \setminus C$. To see it, assume the contrary. Then, there is a path $\rho$ between $A$ and $B$ in $G$ that is $(Z \setminus C)$-open. Moreover, $C$ must occur in $\rho$ because, otherwise, $\rho$ would also be $Z$-open which would contradict the assumption that $A \ci_G B | Z$. However, this implies that $A \nci_G C | Z \setminus C$, which is a contradiction.

Finally, note that $A \ci_G C | Z \setminus C$ and $A \ci_G B | Z \setminus C$ imply $A \ci_{cl(G)} C | Z \setminus C$ and $A \ci_{cl(G)} B | Z \setminus C$ by the induction hypothesis. Then, $A \ci_{cl(G)} B | Z$ by composition and weak union.

\end{description}

\item[Case 2] $A$ and $B$ are in different undirected connectivity components of $G$. Let $A$ be in the undirected connectivity component $K$ of $G$.

\begin{description}
\item[Case 2.1] There is some node $C \in Z$ that is a spouse of $A$. Then, $B \ci_G C | Z \setminus C$ and, thus, $B \ci_{cl(G)} C | Z \setminus C$ by the induction hypothesis. Note that $A \ci_G B | Z \setminus C$ by Lemma \ref{lem:aux} and, thus, $A \ci_{cl(G)} B | Z \setminus C$ by the induction hypothesis. Then, $A \ci_{cl(G)} B | Z$ by symmetry, composition and weak union.

\item[Case 2.2] There is some node $C \in Z$ that is in $K$ such that $C$ has some spouse in $G$ and $A \ci_G C | Z \setminus C$. Then, $A \ci_{cl(G)} C | Z \setminus C$ by the induction hypothesis. Note that $A \ci_G B | Z \setminus C$ by Lemma \ref{lem:aux} and, thus, $A \ci_{cl(G)} B | Z \setminus C$ by the induction hypothesis. Then, $A \ci_{cl(G)} B | Z$ by composition and weak union.

\item[Case 2.3] There is some node $C \in Z$ that is in $K$ such that $C$ has some spouse in $G$ and $A \nci_G C | Z \setminus C$. Then, $B \ci_G C | Z \setminus C$. To see it, assume the contrary. Then, $A \nci_G C | Z \setminus C$ and $B \nci_G C | Z \setminus C$ imply $A \nci_G B | Z \setminus C$ or $A \nci_G B | Z$ by weak transitivity, which implies $A \nci_G B | Z \setminus C$ because $A \ci_G B | Z$ by assumption. However, this contradicts Lemma \ref{lem:aux}.

Finally, note that $B \ci_G C | Z \setminus C$ implies $B \ci_{cl(G)} C | Z \setminus C$ by the induction hypothesis. Note also that $A \ci_G B | Z \setminus C$ by Lemma \ref{lem:aux} and, thus, $A \ci_{cl(G)} B | Z \setminus C$ by the induction hypothesis. Then, $A \ci_{cl(G)} B | Z$ by composition and weak union.

\item[Case 2.4] Cases 2.1-2.3 do not apply. Let $V_1, \ldots, V_m$ be the nodes in $Z$ that are in $K$. Let $W_1, \ldots, W_n$ be the nodes in $Z$ that are not in $K$. Then, $A \ci_G B$, $V_i \ci_G B$, $A \ci_G W_j$, and $V_i \ci_G W_j$ for all $1 \leq i \leq m$ and $1 \leq j \leq n$. Then, $A \ci_{cl(G)} B$, $V_i \ci_{cl(G)} B$, $A \ci_{cl(G)} W_j$, and $V_i \ci_{cl(G)} W_j$ for all $1 \leq i \leq m$ and $1 \leq j \leq n$ by the induction hypothesis. Then, $A \ci_{cl(G)} B | Z$ by symmetry, composition and weak union.

\end{description}

\end{description}

\end{proof}

Let ${\mathcal Q}$ be a partition of $V$. We say that a MCCG $G$ is consistent with ${\mathcal Q}$ if every bidirected edge in $G$ has its end nodes in different elements of ${\mathcal Q}$, and every undirected edge in $G$ has its end nodes in the same element of ${\mathcal Q}$. Note that the elements of ${\mathcal Q}$ may not be undirectly connected in $G$ and, thus, they may not coincide with the undirected connectivity components of $G$. Therefore, every undirected connectivity component of $G$ is contained in some element of ${\mathcal Q}$ but an element of ${\mathcal Q}$ may contain several undirected connectivity components of $G$.

We define the pairwise separation base of a MCCG $G$ relative to a partition ${\mathcal Q}$ of $V$ that is consistent with $G$ as the following set of separations:
\begin{itemize}
\item $A \ci B$ for all $A, B \in V$ such that $A$ and $B$ are not adjacent in $G$ and are in different elements of ${\mathcal Q}$, and
\item $A \ci B | Q \setminus \{A, B\}$ for all $A, B \in V$ such that $A$ and $B$ are not adjacent in $G$ and are in the same element $Q$ of ${\mathcal Q}$.
\end{itemize}

We define the closure of the pairwise separation base of $G$ relative to ${\mathcal Q}$, denoted as $cp(G, {\mathcal Q})$, as the set of separations that are in the base plus those that can be derived from it by applying the graphoid, composition and weak transitivity properties. We denote the separations in $cp(G, {\mathcal Q})$ as $X \ci_{cp(G, {\mathcal Q})} Y | Z$.

\begin{theorem}\label{the:pairwise}
For any MCCG $G$ and any partition ${\mathcal Q}$ of $V$ that is consistent with $G$, $X \ci_{cl(G)} Y | Z$ if and only if $X \ci_{cp(G, {\mathcal Q})} Y | Z$.
\end{theorem}

\begin{proof}
It suffices to prove that the separations in the local (respectively pairwise) separation base are in the closure of the pairwise (respectively local) separation base.

Let $A$ belong to the element $Q$ of ${\mathcal Q}$. Let $A$ belong to the undirected connectivity component $K$ of $G$. Recall that $K \subseteq Q$. Let $V_1, \ldots, V_l$ denote the nodes in $ne_G(A)$. Let $V_{l+1}, \ldots, V_m$ denote the nodes in $K \setminus ne_G(A) \setminus A$. Let $V_{m+1}, \ldots, V_n$ denote the nodes in $Q \setminus K$. Then, $A \ci V_i | Q \setminus \{A, V_i\}$ is in the pairwise separation base of $G$ for all $l+1 \leq i \leq n$. Then, $A \ci \{V_{l+1}, \ldots, V_n\} | \{V_1, \ldots, V_l\}$ is in $cp(G, {\mathcal Q})$ by intersection and, thus, $A \ci V_i | \{V_1, \ldots, V_l\}$ is in $cp(G, {\mathcal Q})$ by decomposition for all $l+1 \leq i \leq n$. Consequently, the separations in the local separation base are in the closure of the pairwise separation base.

Likewise, note that $A \ci V_j | \{V_1, \ldots, V_l\}$ is in the local separation base of $G$ for all $l+1 \leq j \leq m$. Note also that there is no bidirected edge in $G$ between any two nodes in ${\mathcal Q}$, because $G$ is consistent with ${\mathcal Q}$. Therefore, $A \ci V_k$ and $V_i \ci V_k$ are in the local separation base of $G$ for all $1 \leq i \leq l$ and $m+1 \leq k \leq n$. Then, $A \ci \{V_{l+1}, \ldots, V_m\} | \{V_1, \ldots, V_l\}$ is in $cl(G)$ by composition. Moreover, $A \cup \{V_1, \ldots, V_l\} \ci \{V_{m+1}, \ldots, V_n\}$ is in $cl(G)$ by symmetry and composition and, thus, $A \ci \{V_{m+1}, \ldots, V_n\} | \{V_1, \ldots, V_l\}$ is in $cl(G)$ by symmetry and weak union. Then, $A \ci \{V_{l+1}, \ldots, V_n\} | \{V_1, \ldots, V_l\}$ is in $cl(G)$ by composition and, thus, $A \ci V_i | Q \setminus \{A, V_i\}$ is in $cl(G)$ by weak union for all $l+1 \leq i \leq n$. Consequently, the separations in the pairwise separation base are in the closure of the local separation base.
\end{proof}

\begin{corollary}\label{cor:equivalence2}
For any MCCG $G$ and any partition ${\mathcal Q}$ of $V$ that is consistent with $G$, $X \ci_G Y | Z$ if and only if $X \ci_{cl(G)} Y | Z$ if and only if $X \ci_{cp(G, {\mathcal Q})} Y | Z$.
\end{corollary}

\subsection{Markov Equivalence of MCCGs}\label{sec:equivalence}

We say that two MCCGs are Markov equivalent if they induce the same independence model. In a MCCG, a triplex $(\{A, C\},B)$ is an induced subgraph $A \aa B \aa C$, $A \aa B - C$ or $A - B \aa C$. We say that two MCCGs are triplex equivalent if they have the same adjacencies and triplexes.

\begin{theorem}
Two MCCGs are Markov equivalent if and only if they are triplex equivalent.
\end{theorem}

\begin{proof}
We first prove the ``only if" part. Let $G_1$ and $G_2$ be two Markov equivalent MCCGs. First, assume that $G_1$ and $G_2$ do not have the same adjacencies. Specifically, assume without loss of generality that $A$ and $B$ are adjacent in $G_2$ but not in $G_1$. Then, $A \ci_{G_1} B$ or $A \ci_{G_1} B | ne_{G_1}(A)$ by Theorem \ref{the:local1} but neither of the two separations holds in $G_2$, which is a contradiction.

Second, assume that $G_1$ and $G_2$ have the same adjacencies but different triplexes. Specifically, assume without loss of generality that $G_1$ has a triplex $(\{A,C\},B)$ that $G_2$ does not have. Assume also without loss of generality that $A \aa B$ is in $G_1$. Then, $A \ci_{G_1} C$ or $A \ci_{G_1} C | ne_{G_1}(A)$ by Theorem \ref{the:local1} but neither of the two separations holds in $G_2$ because $B \notin ne_{G_1}(A)$, which is a contradiction. 

We now prove the ``if" part. Let $G_1$ and $G_2$ be two triplex equivalent MCCGs. We prove below that $X \nci_{G_1} Y | Z$ implies $X \nci_{G_2} Y | Z$. The opposite implication can be proven in the same manner by just exchanging the roles of $G_1$ and $G_2$ in the proof. Specifically, assume that $X \nci_{G_1} Y | Z$. Let $\rho_1$ be any of the shortest $Z$-open paths between a node in $X$ and a node in $Y$ in $G_1$. Let $\rho_2$ be the path in $G_2$ that consists of the same nodes as $\rho_1$. Then, $\rho_2$ is $Z$-open. To see it, assume the contrary. Then, one of the following cases must occur.

\begin{description}
\item[Case 1] $\rho_2$ does not have a triplex $(\{A,C\},B)$, $B \in Z$, and $B$ has no spouse in $G_2$. Then, one of the following cases must occur.

\begin{description}
\item[Case 1.1] $\rho_1$ has a triplex $(\{A,C\},B)$. Then, $A$ and $C$ must be adjacent in $G_1$ and $G_2$ because, otherwise, $G_1$ and $G_2$ would not be triplex equivalent. Let $\varrho_1$ be the path obtained from $\rho_1$ by replacing the triplex $(\{A,C\},B)$ with the edge between $A$ and $C$ in $G_1$. Note that $\varrho_1$ cannot be $Z$-open because, otherwise, it would contradict the definition of $\rho_1$. Then, $\varrho_1$ is not $Z$-open because $A$ or $C$ do not meet the requirements. Assume without loss of generality that $C$ does not meet the requirements. Then, one of the following cases must occur.

\begin{description}
\item[Case 1.1.1] $\varrho_1$ does not have a triplex $(\{A,D\},C)$, $C \in Z$, and $C$ has no spouse in $G_1$. Then, one of the following subgraphs must occur in $G_1$.

\begin{table}[H]
\centering
\scalebox{0.75}{
\begin{tabular}{ccc}
\begin{tikzpicture}[inner sep=1mm]
\node at (0,0) (A) {$A$};
\node at (1,0) (B) {$B$};
\node at (2,0) (C) {$C$};
\node at (3,0) (D) {$D$};
\path[<->] (A) edge (B);
\path[<->] (B) edge (C);
\path[-] (C) edge (D);
\path[-] (A) edge [bend left] (C);
\end{tikzpicture}
&
\begin{tikzpicture}[inner sep=1mm]
\node at (0,0) (A) {$A$};
\node at (1,0) (B) {$B$};
\node at (2,0) (C) {$C$};
\node at (3,0) (D) {$D$};
\path[<->] (A) edge (B);
\path[-] (B) edge (C);
\path[-] (C) edge (D);
\path[-] (A) edge [bend left] (C);
\end{tikzpicture}
&
\begin{tikzpicture}[inner sep=1mm]
\node at (0,0) (A) {$A$};
\node at (1,0) (B) {$B$};
\node at (2,0) (C) {$C$};
\node at (3,0) (D) {$D$};
\path[-] (A) edge (B);
\path[<->] (B) edge (C);
\path[-] (C) edge (D);
\path[-] (A) edge [bend left] (C);
\end{tikzpicture}
\end{tabular}}
\end{table}

However, the first and third subgraphs imply a contradiction, because $C$ has some spouse in $G_1$. The second subgraph also implies a contradiction, because $\rho_1$ is not $Z$-open.

\item[Case 1.1.2] $\varrho_1$ has a triplex $(\{A,D\},C)$ and $C \notin Z$. Note that $C$ cannot be a triplex node in $\rho_1$ because, otherwise, $\rho_1$ would not be $Z$-open. Then, the following subgraph must occur in $G_1$.

\begin{table}[H]
\centering
\scalebox{0.75}{
\begin{tabular}{c}
\begin{tikzpicture}[inner sep=1mm]
\node at (0,0) (A) {$A$};
\node at (1,0) (B) {$B$};
\node at (2,0) (C) {$C$};
\node at (3,0) (D) {$D$};
\path[<->] (A) edge (B);
\path[-] (B) edge (C);
\path[-] (C) edge (D);
\path[<->] (A) edge [bend left] (C);
\end{tikzpicture}
\end{tabular}}
\end{table}

Moreover, the subgraph above implies that the edge $B - D$ must be in $G_1$ by definition of MCCG. Then, the path obtained from $\rho_1$ by replacing the subpath $B - C - D$ with the edge $B - D$ is $Z$-open. However, this contradicts the definition of $\rho_1$.

\end{description}

\item[Case 1.2] $\rho_1$ does not have a triplex $(\{A,C\},B)$. Then, $B$ must have some spouse in $G_1$, because $B \in Z$ and $\rho_1$ is $Z$-open. Then, the edge $A - C$ must be in $G_1$ by definition of MCCG. Then, the path obtained from $\rho_1$ by replacing the subpath $A - B - C$ with the edge $A - C$ is $Z$-open. However, this contradicts the definition of $\rho_1$.

\end{description}
\item[Case 2] $\rho_2$ has a triplex $(\{A,C\},B)$ and $B \notin Z$. Then, $\rho_1$ does not have a triplex $(\{A,C\},B)$ because, otherwise, $\rho_1$ would not be $Z$-open. Then, $A$ and $C$ must be adjacent in $G_1$ and $G_2$ because these are triplex equivalent. Let $\varrho_1$ be the path obtained from $\rho_1$ by replacing the triplex $(\{A,C\},B)$ with the edge between $A$ and $C$ in $G_1$. Note that $\varrho_1$ cannot be $Z$-open because, otherwise, it would contradict the definition of $\rho_1$. Then, $\varrho_1$ is not $Z$-open because $A$ or $C$ do not meet the requirements. Assume without loss of generality that $C$ does not meet the requirements. Then, one of the following cases must occur.

\begin{description}
\item[Case 2.1] $\varrho_1$ does not have a triplex $(\{A,D\},C)$, $C \in Z$, and $C$ has no spouse in $G_1$. Then, the following subgraph must occur in $G_1$.

\begin{table}[H]
\centering
\scalebox{0.75}{
\begin{tabular}{c}
\begin{tikzpicture}[inner sep=1mm]
\node at (0,0) (A) {$A$};
\node at (1,0) (B) {$B$};
\node at (2,0) (C) {$C$};
\node at (3,0) (D) {$D$};
\path[-] (A) edge (B);
\path[-] (B) edge (C);
\path[-] (C) edge (D);
\path[-] (A) edge [bend left] (C);
\end{tikzpicture}
\end{tabular}}
\end{table}

However, this subgraph implies that $\rho_1$ is not $Z$-open, which is a contradiction.

\item[Case 2.2] $\varrho_1$ has a triplex $(\{A,D\},C)$ and $C \notin Z$. Note that $C$ cannot be a triplex node in $\rho_1$ because, otherwise, $\rho_1$ would not be $Z$-open. Then, the following subgraph must occur in $G_1$.

\begin{table}[H]
\centering
\scalebox{0.75}{
\begin{tabular}{c}
\begin{tikzpicture}[inner sep=1mm]
\node at (0,0) (A) {$A$};
\node at (1,0) (B) {$B$};
\node at (2,0) (C) {$C$};
\node at (3,0) (D) {$D$};
\path[-] (A) edge (B);
\path[-] (B) edge (C);
\path[-] (C) edge (D);
\path[<->] (A) edge [bend left] (C);
\end{tikzpicture}
\end{tabular}}
\end{table}

Moreover, the subgraph above implies that the edge $B - D$ must be in $G_1$ by definition of MCCG. Then, the path obtained from $\rho_1$ by replacing the subpath $B - C - D$ with the edge $B - D$ is $Z$-open. However, this contradicts the definition of $\rho_1$.

\end{description}

\end{description}

\end{proof}

It is worth mentioning that the proof of the theorem above only makes use of concepts introduced in this paper. An alternative proof of the theorem above that relies upon previous works is as follows. As we will note later in Proposition \ref{pro:mags}, every MCCG can be transformed into a maximal ancestral graph \citep{RichardsonandSpirtes2002} that induces the same independence model as the MCCG. Moreover, the Markov equivalence of maximal ancestral graphs has been characterized \citep[Theorem 4.1]{Alietal.2009}. It follows from this characterization that two maximal ancestral graphs obtained from two MCCGs via Proposition \ref{pro:mags} are Markov equivalent if and only if the two MCCGs are triplex equivalent.

\begin{lemma}
For every triplex equivalence class of MCCGs, there is a unique maximal (with respect to set inclusion) set of bidirected edges such that some MCCG in the class has exactly those bidirected edges.
\end{lemma}

\begin{proof}
Assume to the contrary that there are two such sets of bidirected edges. Let the MCCG $G$ have exactly the bidirected edges in one of the sets, and let the MCCG $H$ have exactly the bidirected edges in the other set. For every edge $A \aa B$ in $G$ such that $A - B$ is in $H$, replace $A - B$ with $A \aa B$ in $H$ and call the resulting graph $F$. We prove below that $F$ is a MCCG that is triplex equivalent to $G$, which is a contradiction since $F$ has a proper superset of the bidirected edges in $G$.

First, we show that $F$ has no induced subgraph $A - C - B$ such that $C$ has some spouse $D$ in $F$. Assume the contrary. Then, the induced subgraph $A - C - B$ must occur in $G$ and $H$. Moreover, the edge $C \aa D$ must be in $G$ or $H$. Then, $G$ or $H$ has an induced subgraph $A - C - B$ plus the edge $C \aa D$, which contradicts the definition of MCCG.

Second, we show that $F$ has no cycle $A - \ldots - B \aa A$. Assume the contrary. Then, the subgraph $A - \ldots - B$ must occur in $G$ and $H$. Moreover, the edge $B \aa A$ must be in $G$ or $H$. Then, the cycle $A - \ldots - B \aa A$ must occur in $G$ or $H$, which contradicts the definition of MCCG.

Third, note that $F$ has the same adjacencies as $G$. Fourth, note that all the triplexes in $G$ are in $F$ too. Finally, assume to the contrary that $F$ has a triplex $(\{A,C\},B)$ that $G$ does not have (and, thus, nor does $H$). Then, the subgraph $A - B - C$ must be in $G$ and $H$. However, this implies that the subgraph $A - B - C$ is in $F$, which is a contradiction.
\end{proof}

Note that the theorem above does not hold if the word maximal is replaced by minimal. A simple counterexample is the triplex equivalence class that contains the MCCGs $A \aa B - C$ and $A - B \aa C$.

We say that a MCCG $G$ is blarger than another MCCG $H$ if every bidirected edge in $H$ is in $G$. The lemma above implies that every triplex equivalence class has a distinguished member, namely the blargest MCCG in the class. We show below how this distinguished member can be obtained from any other member of the class. By bidirecting an undirected connectivity component $K$ of a MCCG $G$, we mean replacing every edge $A - B$ in $G$ such that $A, B \in K$ with an edge $A \aa B$. Moreover, we say that the bidirecting is feasible if $K$ is a complete set.

\begin{lemma}\label{lem:bidirecting}
The graph $H$ resulting from performing a feasible bidirecting on a MCCG $G$ is a MCCG that is triplex equivalent to $G$.
\end{lemma}

\begin{proof}
Let $K$ denote the undirected connectivity component of $G$ that got bidirected. First, we show that $H$ has no induced subgraph $A - C - B$ such that $C$ has some spouse $D$ in $H$. Assume the contrary. Then, $C, D \in K$ because, otherwise, $G$ would not be a MCCG. Therefore, $A, B, C, D \in K$ and, thus, the edges $A \aa C$ and $C \aa B$ must be in $H$, which is a contradiction.

Second, we show that $H$ has no cycle $V_1 - \ldots - V_n \aa V_1$. Assume the contrary. Then, $V_1, V_n \in K$ because, otherwise, $G$ would not be a MCCG. Therefore, $V_i \in K$ for all $1 \leq i \leq n$ and, thus, the edge $V_i \aa V_{i+1}$ must be in $H$ for all $1 \leq i < n$, which is a contradiction.

Third, note that $H$ has the same adjacencies as $G$. Fourth, note that all the triplexes in $G$ are in $H$ too. Finally, assume to the contrary that $H$ has a triplex $(\{A,C\},B)$ that $G$ does not have. Then, the induced subgraph $A - C - B$ must be in $G$ and, thus, $A, B, C \in K$. However, this implies that $K$ is not a complete set, which is a contradiction.
\end{proof}

\begin{table}[t]
\caption{Algorithm for learning MCCGs.}\label{tab:algorithm2}
\centering
\scalebox{0.8}{
\begin{tabular}{rl}
\\
\hline
\\
& Input: A probability distribution $p$ that is faithful to an unknown MCCG $G$.\\
& Output: The blargest MCCG $H$ in the triplex equivalent class of $G$.\\
\\
1 & Let $H$ denote the complete bidirected graph\\
2 & Set $l=0$\\
3 & Repeat while $l \leq |V|-2$\\
4 & \hspace{0.2cm} For each ordered pair of nodes $A$ and $B$ in $H$ such that $A \in ad_H(B)$ and $|ad_H(A) \setminus B| \geq l$\\
5 & \hspace{0.5cm} If there is some $S \subseteq ad_H(A) \setminus B$ such that $|S|=l$ and $A \ci_p B | S$ then\\
6 & \hspace{0.8cm} Set $S_{AB}=S_{BA}=S$\\
7 & \hspace{0.8cm} Remove the edge $A \aa B$ from $H$\\
8 & \hspace{0.2cm} Set $l=l+1$\\
9 & Replace every induced subgraph $A \aa B \aa C$ in $H$ such that $B \in S_{AC}$ with $A - B - C$\\
10 & If there is an edge $A \aa B$ in $H$ that violates the constraint C1 or C2 then\\
11 & \hspace{0.2cm} Replace the edge $A \aa B$ in $H$ with $A - B$\\
12 & \hspace{0.2cm} Go to line 10\\
\\
\hline
\\
\end{tabular}}
\end{table}

\begin{lemma}\label{lem:nomore}
If no feasible bidirecting can be performed on a MCCG $G$, then $G$ is the blargest MCCG in its triplex equivalence class.
\end{lemma}

\begin{proof}
Assume to the contrary that $H$ and not $G$ is the blargest MCCG in the triplex equivalence class of $G$. Then, there must exist an edge $A - B$ in $G$ such that the edge $A \aa B$ is in $H$. Let $K$ denote the undirected connectivity component of $G$ such that $A, B \in K$. Note that $K$ cannot be a complete set, because no feasible bidirecting can be performed on $G$. Then, $G$ has an induced subgraph $V_1 - V_2 - V_3$ with $V_1, V_2, V_3 \in K$. Note that $H$ also has an induced subgraph $V_1 - V_2 - V_3$ because, otherwise, $G$ and $H$ would not be triplex equivalent. Then, $G$ must have a subgraph $V_1 - V_2 - V_3 - \ldots - V_{n-1} - V_n$ such that $V_1 - V_2 - V_3 - \ldots - V_{n-1} \aa V_n$ is a subgraph of $H$. Note that $V_{n-2}$ and $V_n$ must be adjacent in $G$ and $H$ because, otherwise, $G$ and $H$ would not be triplex equivalent. Then, the edge $V_{n-2} - V_n$ (respectively $V_{n-2} \aa V_n$) must be in $G$ (respectively $H$) by definition of MCCG. Likewise, the edge $V_i - V_n$ (respectively $V_i \aa V_n$) must be in $G$ (respectively $H$) for all $1 \leq i \leq n-3$. However, this implies that $V_1 - V_n - V_3$ is an induced subgraph of $G$ whereas $V_1 \aa V_n \aa V_3$ is an induced subgraph of $H$, which contradicts the assumption that $G$ and $H$ are triplex equivalent.
\end{proof}

\begin{theorem}\label{the:equivalence}
The blargest MCCG in a triplex equivalence class of MCCGs can be obtained from any member of the class by performing feasible bidirectings until no more can be performed.
\end{theorem}

\begin{proof}
It follows from Lemmas \ref{lem:bidirecting} and \ref{lem:nomore}.
\end{proof}

By undirecting a set of bidirected edges in a MCCG, we mean the inverse operation of bidirecting an undirected connectivity component of a MCCG. In other words, the result of undirecting a set of bidirected edges in a MCCG $H$ is a MCCG $G$ such that the result of bidirecting an undirected connectivity component in $G$ is $H$. Moreover, we say that the undirecting is feasible if the corresponding bidirecting is feasible.

\begin{corollary}\label{cor:equivalence}
Any member of a triplex equivalence class of MCCGs can be obtained from any other member of the class by performing a sequence of feasible bidirectings and undirectings.
\end{corollary}

\subsection{Algorithm for Learning MCCGs}\label{sec:algorithm}

In this section, we present a constraint based algorithm for learning a MCCG a given probability distribution is faithful to. The algorithm, which can be seen in Table \ref{tab:algorithm2}, resembles the well-known PC algorithm \citep{Spirtesetal.1993}. It consists of two phases: The first phase (lines 1-8) aims at learning the adjacencies, whereas the second phase (lines 9-12) aims at learning the edge type for each adjacency learnt. Specifically, the first phase declares that two nodes are adjacent if and only if they are not separated by any set of nodes. Note that the algorithm does not test every possible separator (see line 5). Note also that the separators tested are tested in increasing order of size (see lines 2, 5 and 8). The second phase identifies the edge type for each pair of adjacent nodes by avoiding false triplexes (line 9) and enforcing the constraints C1 and C2 (lines 10-12).

\begin{theorem}
After line 12, $H$ is the blargest MCCG in the triplex equivalent class of $G$.
\end{theorem}

\begin{proof}
First, we prove that $G$ and $H$ have the same adjacencies after line 8. Consider any pair of nodes $A$ and $B$ in $G$. If $A \in ad_G(B)$, then $A \nci_p B | S$ for all $S \subseteq V \setminus \{A, B\}$ by the faithfulness assumption. Consequently, $A \in ad_H(B)$ at all times. On the other hand, if $A \notin ad_G(B)$, then $A \ci_p B$ or $A \ci_p B | ne_G(A)$ by the faithfulness assumption and Theorem \ref{the:local1}. Note that, as mentioned before, $ne_G(A) \subseteq ad_H(A) \setminus B$ at all times. Therefore, there will exist some $S$ in line 5 such that $A \ci_p B | S$ and, thus, the edge $A \aa B$ will be removed from $H$ in line 7. Consequently, $A \notin ad_H(B)$ after line 8.

Second, $G$ and $H$ must be triplex equivalent after line 12 because, as shown above, they have the same adjacencies and lines 9-12 perform only necessary replacements. Actually, for the same reason, $H$ must be the blargest MCCG in the triplex equivalent class of $G$.
\end{proof}

\subsection{Discussion}\label{sec:discussion2}

This section has aimed at solving the problem of how to represent the result of marginalizing out some nodes in an AMP CG. We have introduced maximal covariance-concentration graphs (MCCGs), a new family of graphical models that solves this problem partially. However, if we forget for a moment our motivation to develop MCCGs and treat AMP CGs and MCCGs as two competing families of graphical models, then one may want to know when one is more suitable than the other. For instance, AMP CGs may be preferred when a causal order of the nodes exists. This heuristic is perfectly reasonable but it may fail if the order is partial. For instance, consider the AMP CG $A \ra B \la C \ra D \la E \ra F \la G$. Marginalize out the nodes $C$ and $E$. Then, the resulting independence model can be represented by the MCCG $A \aa B \aa D \aa F \aa G$, but it cannot be represented by any AMP CG despite the existence of the partial order $\{A < B, F < G\}$. On the other hand, MCCGs may be preferred when latent variables exist. Again, this heuristic is perfectly reasonable but it may fail. For instance, consider the AMP CG $A \ra B \ra C \ra D \la E \la F \la G$. Marginalize out the nodes $B$ and $F$. Then, the resulting independence model can be represented by the AMP CG $A \ra C \ra D \la E \la G$, but it cannot be represented by any MCCG despite the existence of the latent variables $B$ and $F$. In summary, these two examples show that there are independence models that can be represented by one family and not by the other. Therefore, we believe that it is more beneficial to see these two families as complementary rather than as competing.

It is also worth assessing the merits of MCCGs with respect to other families of graphical models such as maximal ancestral graphs, summary graphs and MC graphs. A maximal ancestral graph (MAG) is a graph whose every edge is undirected, directed or bidirected, and that satisfies certain topological constraints \citep{RichardsonandSpirtes2002}. Among the topological constraints, only the following is relevant in this paper: A MAG cannot have a subgraph of the form $A \oa B - C$, where the circle represents an unspecified end, i.e. an arrow tip or nothing. This constraint clearly implies that not every MCCG is a MAG. However, every independence model induced by a MCCG can be induced by a MAG, as the proposition below shows. Therefore, in this sense, MCCGs are a subfamily of MAGs. Before we can state the mentioned proposition, we need to introduce the separation criterion for MAGs. A route $V_{1}, \ldots, V_{n}$ in a MAG $G$ is called strictly descending if $V_i \ra V_{i+1}$ is in $G$ for all $1 \leq i < n$. The strict ascendants of a set of nodes $X$ is the set $san_G(X) = \{V_1 |$ there is a strictly descending route from $V_1$ to $V_n$ in $G$, $V_1 \notin X$ and $V_n \in X \}$. A node $B$ in a path $\rho$ in $G$ is called a triplex node in $\rho$ if $A \oa B \ao C$ is a subpath of $\rho$. Let $X$, $Y$ and $Z$ denote three pairwise disjoint subsets of $V$. A path $\rho$ in $G$ is said to be $Z$-open when (i) every triplex node in $\rho$ is in $Z \cup san_G(Z)$, and (ii) every non-triplex node $B$ in $\rho$ is outside $Z$. When there is no $Z$-open path in $G$ between a node in $X$ and a node in $Y$, we say that $X$ is separated from $Y$ given $Z$ in $G$ and denote it as $X \ci_G Y | Z$. The independence model induced by $G$ is the set of separation statements $X \ci_G$ $Y | Z$.

\begin{proposition}\label{pro:mags}
Every MCCG can be translated into a Markov equivalent MAG by just replacing every subgraph $A \aa B - C$ by $A \aa B \la C$.
\end{proposition}

Note that the replacement in the proposition above may create new bidirected edges. For instance, the MCCG $A \aa B - C \aa D$ gets translated into the MAG $A \aa B \aa C \aa D$. Despite the proposition above, there are cases where a MCCG is a more natural representation of the domain at hand than a MAG and, thus, MCCGs are still worth consideration. The following example illustrates this.

\begin{example}
Consider the AMP CG $A \la B \ra C - D - E$ and call it $G$. Consider the independence model resulting from $G$ by marginalizing out $B$. This model can be represented by the MCCG $A \aa C - D - E$, or by the MAGs $A \oa C \la D \oo E$ or $A \oa C \aa D \ra E$. Note that these are all the MAGs that can represent the model. However, the MAGs suggest the existence of the causal relationship $C \la D$ or $D \ra E$, although neither exists in $G$. On the other hand, the MCCG does not suggest any causal relationship and, thus, it is preferable.
\end{example}

The reason why MAGs conflict with the original model in the example above is that MAGs were introduced to represent the result of marginalization and/or conditioning in directed acyclic graphs, not in AMP CGs. Recall that MCCGs have been introduced to represent the result of marginalization in certain AMP CGs. Two other families of graphical models that induce all the independence models induced by MAGs (and thus by MCCGs) are summary graphs \citep{CoxandWermuth1996} and MC graphs \citep{Koster2002}. However, these families have a rather counterintuitive and undesirable feature: Not every missing edge corresponds to a separation \citep[p. 1023]{RichardsonandSpirtes2002}. MCCGs and MAGs, on the other hand, do not have this disadvantage (see, respectively, Theorem \ref{the:local1} and \citep[Corollary 4.19]{RichardsonandSpirtes2002}).

At the beginning of this section, we have noted that the new family includes both covariance and concentration graphs as subfamilies. Thus, it allows to model the covariance and concentration matrices of a Gaussian probability distribution jointly by a single graph, rather than modeling the former by a covariance graph and the latter by a concentration graph. We have argued that, by doing so, the new family may model more accurately the probability distribution. We show below an example that illustrates this.

\begin{example}
Consider a Gaussian probability distribution $p$ that is faithful to the MCCG $G$ below. Recall from Theorem \ref{the:faithfulness} that such a probability distribution exists.

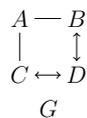
\begin{table}[H]
\centering
\scalebox{0.75}{
\begin{tabular}{c}
\begin{tikzpicture}[inner sep=1mm]
\node at (0,0) (A) {$A$};
\node at (1,0) (B) {$B$};
\node at (0,-1) (C) {$C$};
\node at (1,-1) (D) {$D$};
\path[-] (A) edge (B);
\path[-] (A) edge (C);
\path[<->] (B) edge (D);
\path[<->] (C) edge (D);
\end{tikzpicture}
\\
$G$
\end{tabular}}
\end{table}

The covariance graph and the concentration graph of $p$ are depicted by the graphs $H$ and $F$ below.

\begin{table}[H]
\centering
\scalebox{0.75}{
\begin{tabular}{cc}
\begin{tikzpicture}[inner sep=1mm]
\node at (0,0) (A) {$A$};
\node at (1,0) (B) {$B$};
\node at (0,-1) (C) {$C$};
\node at (1,-1) (D) {$D$};
\path[<->] (A) edge (B);
\path[<->] (A) edge (C);
\path[<->] (B) edge (C);
\path[<->] (B) edge (D);
\path[<->] (C) edge (D);
\end{tikzpicture}
&
\begin{tikzpicture}[inner sep=1mm]
\node at (0,0) (A) {$A$};
\node at (1,0) (B) {$B$};
\node at (0,-1) (C) {$C$};
\node at (1,-1) (D) {$D$};
\path[-] (A) edge (B);
\path[-] (A) edge (C);
\path[-] (A) edge (D);
\path[-] (B) edge (C);
\path[-] (B) edge (D);
\path[-] (C) edge (D);
\end{tikzpicture}
\\
$H$ & $F$
\end{tabular}}
\end{table}

Now, note that $B \ci_p C | A$ because $B \ci_G C | A$. However, $B \nci_H C | A$ and $B \nci_F C | A$.
\end{example}

Finally, we briefly describe below some extensions to the work presented in this section that we are currently exploring.

\begin{itemize}
\item Despite the example above, we do not discard the possibility that some Gaussian probability distributions are modeled more accurately by a covariance graph plus a concentration graph than by a MCCG. We would like to study when this occurs, if at all.

\item We would like to remove the constraint that MCCGs are simple graphs to allow the possibility of having an undirected and a bidirected edge between two nodes.

\item We would like to extend MCCGs with directed edges, so that they can represent the result of marginalization and/or conditioning in AMP CGs.

\item We would like to find an efficient parameterization of MCCGs, and a factorization rule for the probability distributions that satisfy the independencies represented by a MCCG.

\item The correctness of our learning algorithm lies upon the assumption that $p$ is faithful to some MCCG. This is a strong requirement that we would like to weaken, e.g. by replacing it with the milder assumption that $p$ satisfies the composition property. However, as with AMP CGs (recall Section \ref{sec:discussion}), the extension of Meek's conjecture to MCCGs does not hold, as the example below illustrates. This compromises the development of score+search learning algorithms that are correct and efficient under the composition property assumption. It is not clear to us whether it also does it for constraint based algorithms. This is something we plan to study.
\end{itemize}

\begin{example}
Consider the MCCGs $F$ and $H$ below.

\begin{table}[H]
\centering
\scalebox{0.75}{
\begin{tabular}{cc}
\begin{tikzpicture}[inner sep=1mm]
\node at (1,0) (A) {$A$};
\node at (-1,-1) (B) {$B$};
\node at (0,-1) (C) {$C$};
\node at (1,-1) (D) {$D$};
\path[<->] (A) edge (D);
\path[-] (B) edge (C);
\path[-] (C) edge (D);
\end{tikzpicture}
&
\begin{tikzpicture}[inner sep=1mm]
\node at (1,0) (A) {$A$};
\node at (-1,-1) (B) {$B$};
\node at (0,-1) (C) {$C$};
\node at (1,-1) (D) {$D$};
\path[-] (A) edge (D);
\path[-] (A) edge (C);
\path[-] (B) edge (C);
\path[-] (C) edge (D);
\end{tikzpicture}\\
$F$ & $H$
\end{tabular}}
\end{table}

Then, $I(H)=\{B \ci_H A | C, B \ci_H A | \{C, D\}, B \ci_H D | C, B \ci_H D | \{C, A\}, B \ci_H \{A, D\} | C \}$. One can easily confirm by using the definition of separation that $I(H) \subseteq I(F)$. One can also confirm by using Corollary \ref{cor:equivalence} that there is no MCCG that is triplex equivalent to $F$ or $H$. Finally, it is obvious that one cannot transform $F$ into $H$ by adding a single edge.
\end{example}

\section{Identifying (In)Dependencies from MCCGs}\label{sec:dependencies}

In this section, we present a graphical criterion for reading dependencies from a MCCG $G$ of a probability distribution $p$, under the assumption that $G$ satisfies some topological constraints and $p$ satisfies the graphoid properties, weak transitivity and composition. We prove that the criterion is sound and complete in certain sense.

A MCCG of a WTC graphoid $p$ is a MCCG $G$ such that
\begin{itemize}
\item the edge $A \aa B$ is not in $G$ only if $A \ci_p B$, and
\item the edge $A - B$ is not in $G$ only if $A \ci_p B | K \setminus \{A, B\}$, where $K$ denotes the undirected connectivity component of $G$ that contains $A$ and $B$.
\end{itemize}

Note that the separation criterion introduced in Section \ref{sec:mccg} is sound and complete for identifying independencies in $p$ from $G$: It is sound in the sense that it only identifies (true) independencies in $p$, and it is complete in the sense it identifies all the independencies in $p$ that can be identified by studying $G$ alone. Soundness follows as follows. Recall from Corollary \ref{cor:equivalence2} that the separations identified in $G$ by this graphical criterion correspond with those in $cp(G, {\mathcal Q})$ for any partition ${\mathcal Q}$ of $V$ that is consistent with $G$. Specifically, let ${\mathcal Q}$ denote the undirected connectivity components of $G$. Then, $p$ satisfies the independencies corresponding to the separations in the pairwise separation base of $G$ relative to ${\mathcal Q}$, by definition of $G$. Thus, $p$ satisfies the independencies corresponding to the separations in $cp(G, {\mathcal Q})$, because $p$ is a WTC graphoid. Completeness follows from the fact that there are WTC graphoids that are faithful to $G$ (Theorem \ref{the:faithfulness} and Corollary \ref{cor:wtc}) and, thus, $p$ may be one of them (whether $p$ is really faithful to $G$ is impossible to know on the sole basis of $G$).

Note that every edge in a MCCG of a WTC graphoid does not correspond to a dependence. Since this may be undesirable in some cases, we strengthen the definition above as follows. A minimal MCCG (MMCCG) of a WTC graphoid $p$ is a MCCG $G$ such that
\begin{itemize}
\item the edge $A \aa B$ is not in $G$ if and only if $A \ci_p B$, and
\item the edge $A - B$ is not in $G$ if and only if $A \ci_p B | K \setminus \{A, B\}$, where $K$ denotes the undirected connectivity component of $G$ that contains $A$ and $B$.
\end{itemize}

Note that, by Corollary \ref{cor:equivalence2}, we can alternatively define that a MCCG (respectively MMCCG) of a WTC graphoid $p$ is a MCCG $G$ such that
\begin{itemize}
\item the edge $A \aa B$ is not in $G$ only if (respectively if and only if) $A \ci_p B$, and
\item the edge $A - B$ is not in $G$ only if (respectively if and only if) $A \ci_p B | ne_G(A)$.
\end{itemize}

An interesting feature of a MMCCG $G$ of a WTC graphoid $p$ is that it allows us to identify not only independencies in $p$ as shown above but also dependencies in $p$ as we show below. Specifically, we introduce below a sound and complete graphical criterion for identifying dependencies in $p$ from $G$, under the assumption that $G$ has no cycle with both undirected and bidirected edges. This assumption implies that the connectivity components of $G$ form a kind of tree, as the following example illustrates. The remark below formalizes this observation. Note that both covariance and concentrations graphs always satisfy this assumption.

\begin{table}[H]
\centering
\scalebox{0.75}{
\begin{tabular}{c}
\begin{tikzpicture}[inner sep=1mm]
\node at (0,0) (A) {$A$};
\node at (0,-1) (B) {$B$};
\node at (1,0) (C) {$C$};
\node at (2,0) (D) {$D$};
\node at (1,-1) (E) {$E$};
\node at (3,0) (F) {$F$};
\path[<->] (A) edge (C);
\path[<->] (B) edge (E);
\path[-] (C) edge (D);
\path[-] (C) edge (E);
\path[-] (D) edge (E);
\path[<->] (D) edge (F);
\end{tikzpicture}
\end{tabular}}
\end{table}

\begin{remark}\label{rem:tree}
Assume that $G$ has no cycle with both undirected and bidirected edges. Let $K_u$ be any undirected connectivity component of $G$. Let $K_b$ be any bidirected connectivity component of $G$. Then, $K_u \cap K_b$ contains at most one node. Moreover, if $K_u \cap K_b$ contains the node $A$, then every path between a node in $K_u$ and a node in $K_b$ passes through $A$.
\end{remark}

Given a MMCCG $G$ of a WTC graphoid $p$, we know that the following dependencies hold in $p$ by definition of $G$:
\begin{itemize}
\item $A \nci_p B$ for every edge $A \aa B$ in $G$, and
\item $A \nci_p B | K \setminus \{A, B\}$ for every edge $A - B$ in $G$, where $K$ denotes the undirected connectivity component of $G$ that contains $A$ and $B$.
\end{itemize}

We call these dependencies the dependence base of $p$. Further dependencies in $p$ can be derived from the dependence base via the WTC graphoid properties. For this purpose, we rephrase the WTC graphoid properties in their contrapositive form as follows. Symmetry $Y \nci_p X | Z \Rightarrow X \nci_p Y | Z$. Decomposition $X \nci_p Y | Z \Rightarrow X \nci_p Y \cup W | Z$. Weak union $X \nci_p Y | Z \cup W \Rightarrow X \nci_p Y \cup W | Z$. Contraction $X \nci_p Y \cup W | Z \Rightarrow X \nci_p Y | Z \cup W \lor X \nci_p W | Z$ is problematic for deriving new dependencies because it contains a disjunction in the consequent and, thus, we split it into two properties: Contraction1 $X \nci_p Y \cup W | Z \land X \ci_p Y | Z \cup W \Rightarrow X \nci_p W | Z$, and contraction2 $X \nci_p Y \cup W | Z \land X \ci_p W | Z \Rightarrow X \nci_p Y | Z \cup W$. Likewise, intersection gives rise to intersection1 $X \nci_p Y \cup W | Z \land X \ci_p Y | Z \cup W \Rightarrow X \nci_p W | Z \cup Y$, and intersection2 $X \nci_p Y \cup W | Z \land X \ci_p W | Z \cup Y \Rightarrow X \nci_p Y | Z \cup W$. Note that intersection1 and intersection2 are equivalent and, thus, we refer to them simply as intersection. Similarly, weak transitivity gives rise to weak transitivity1 $X \nci_p K | Z \land K \nci_p Y | Z \land X \ci_p Y | Z \Rightarrow X \nci_p Y | Z \cup K$, and weak transitivity2 $X \nci_p K | Z \land K \nci_p Y | Z \land X \ci_p Y | Z \cup K \Rightarrow X \nci_p Y | Z$. Finally, composition $X \nci_p Y \cup W | Z \Rightarrow X \nci_p Y | Z \lor X \nci_p W | Z$ gives rise to composition1 $X \nci_p Y \cup W | Z \land X \ci_p Y | Z \Rightarrow X \nci_p W | Z$, and composition2 $X \nci_p Y \cup W | Z \land X \ci_p W | Z \Rightarrow X \nci_p Y | Z$. Since composition1 and composition2 are equivalent, we refer to them simply as composition. The independence in the antecedent of any of the properties above holds if the corresponding separation holds in $G$. This is the best solution we can hope for because, as shown above, the separation criterion is sound and complete for WTC graphoids. Moreover, this solution does not require more information than what it is available, namely $G$ or equivalently the dependence base of $p$. We define the WTC graphoid closure of the dependence base of $p$ as the set of dependencies that are in the dependence base of $p$ plus those that can be derived from it by applying the nine properties above. Note that we can alternatively define the dependence base of $p$ as the following dependencies and the results below would still hold \cite[p. 1083]{Pennaetal.2009}:
\begin{itemize}
\item $A \nci_p B$ for every edge $A \aa B$ in $G$, and
\item $A \nci_p B | ne_G(A) \setminus B$ for every edge $A - B$ in $G$.
\end{itemize}

We can now introduce our graphical criterion for identifying dependencies in a WTC graphoid from its MCCG. It is worth mentioning this graphical criterion subsumes those developed by \cite{Pennaetal.2009} and \cite{Penna2013} for reading dependencies from the covariance graph and concentration graph of a WTC graphoid, respectively.

\begin{definition}\label{def:joined}
Let $G$ be the MCCG of a WTC graphoid $p$. Let $X$, $Y$ and $Z$ denote three pairwise disjoint subsets of $V$. We say that $X$ is joined to $Y$ given $Z$ in a MCCG $G$, denoted as $X \de_G Y | Z$, if there exist two nodes $A \in X$ and $B \in Y$ such that there exists a \textbf{single} path $\rho_{A:B}$ between $A$ and $B$ in $G$ that is $U$-open with $Z \subseteq U \subseteq X \cup Y \cup Z \setminus \{A,B\}$. 
\end{definition}

Hereinafter, given a node $C$ in a path $\rho_{A:B}$ between two nodes $A$ and $B$ in a MCCG, we denote by $\rho_{A:C}$ the subpath of $\rho_{A:B}$ between $A$ and $C$.

\begin{remark}\label{rem:noinpath}
In Definition \ref{def:joined}, we can assume without loss of generality that $A$ and $B$ are the only nodes in $\rho_{A:B}$ that are in $X$ and $Y$, respectively.
\end{remark}

\begin{proof}
Let $B' \neq B$ be closest node to $A$ that is in $\rho_{A:B}$ and $Y$. Then, $\rho_{A:B'}$ is the only path between $A$ and $B'$ in $G$ that is $U$-open. To see it, assume to the contrary that there is a second such path $\varrho_{A:B'}$. Note that $\varrho_{A:B'} \cup \rho_{B':B}$ cannot be $U$-open because, otherwise, there would be a second path between $A$ and $B$ in $G$ that is $U$-open, which is a contradiction. Therefore, one of the following cases must occur.

\begin{description}

\item[Case 1] $B'$ is a non-triplex node in $\varrho_{A:B'} \cup \rho_{B':B}$ and $B' \in U$. However, that $B' \in U$ together with the fact that $\rho_{A:B}$ is $U$-open imply that either $B'$ is a triplex node in $\rho_{A:B}$ or $B'$ is a non-triplex node in $\rho_{A:B}$ that has some spouse in $G$. In either case $B'$ has some spouse in $G$ and, thus, $\varrho_{A:B'} \cup \rho_{B':B}$ is $U$-open, which is a contradiction.

\item[Case 2] $B'$ is a triplex node in $\varrho_{A:B'} \cup \rho_{B':B}$ and $B' \notin U$. However, that $B' \notin U$ together with the fact that $\rho_{A:B}$ is $U$-open imply that $B'$ is a non-triplex node in $\rho_{A:B}$. Moreover, that $B'$ is a triplex node in $\varrho_{A:B'} \cup \rho_{B':B}$ implies that $B'$ has some spouse in $G$. Then, removing $B'$ from $\rho_{A:B}$ results in a second path between $A$ and $B$ in $G$ by definition of MCCGs which, moreover, is $U$-open, which is a contradiction.

\end{description}

The proof for $A$ is similar.

\end{proof}

\begin{remark}\label{rem:nospouse}
In Definition \ref{def:joined}, we can assume without loss of generality that the nodes in $U$ that are not in $Z$ or $\rho_{A:B}$ have no spouse in $G$.
\end{remark}

\begin{proof}
Let $C$ be a node that is in $U$ but not in $Z$ or $\rho_{A:B}$. Assume that $C$ has some spouse in $G$. Then, $\rho_{A:B}$ is the only path between $A$ and $B$ in $G$ that is $(U \setminus C)$-open. To see it, assume to the contrary that there is a second such path $\varrho_{A:B}$. Note that $C$ must be a non-triplex node in $\varrho_{A:B}$ because, otherwise, that path would also be $U$-open, which is a contradiction. For the same reason, $C$ cannot have any spouse in $G$. However, this contradicts the assumptions made.
\end{proof}

\begin{remark}
In Definition \ref{def:joined}, we can assume without loss of generality that $U$ contains exactly the nodes in $X \cup Y \cup Z$ that are in $Z$ or $\rho_{A:B}$ or that have no spouse in $G$.
\end{remark}

\begin{proof}
By Remark \ref{rem:nospouse}, we can assume without loss of generality that the nodes in $U$ that are not in $Z$ or $\rho_{A:B}$ have no spouse in $G$. Let $C \in X \cup Y \cup Z$ be a node that is not in $Z$ or $\rho_{A:B}$ and that has no spouse in $G$. Then, $\rho_{A:B}$ is the only path between $A$ and $B$ in $G$ that is $(U \cup C)$-open, because $C$ can neither activate new paths nor deactivate $\rho_{A:B}$. Repeating this reasoning until no such node $C$ exists leads to the desired result.
\end{proof}

The following two theorems prove that the graphical criterion defined above is sound and complete in some sense. We start by proving some auxiliary results.

\begin{lemma}\label{lem:b}
Let $G$ be a MMCCG of a WTC graphoid $p$. Let $A$ and $B$ denote two nodes that are in the same bidirected connectivity component of $G$. If $A \de_G B | U$, then $A \nci_p B | U$ is in the WTC graphoid closure of the dependence base of $p$.
\end{lemma}

\begin{proof}
Let $K$ denote the bidirected connectivity component that contains $A$ and $B$. Let $S_A$ denote the nodes in $U \setminus K$ that are in $ne_G(A)$ or connected to $A$ by a path that passes through $ne_G(A)$. Let $S_B$ denote the nodes in $U \setminus K$ that are in $ne_G(B)$ or connected to $B$ by a path that passes through $ne_G(B)$. Let $S$ denote the nodes in $U \setminus K \setminus S_A \setminus S_B$ that are connected to $A$ or $B$ by a path that passes through $sp_G(A)$ or $sp_G(B)$, respectively.

Note that $A \de_G B | U \cap K$. Note also that the path that makes this statement hold only contains bidirected edges, because all its nodes are in $K$ by Remark \ref{rem:tree}. Thus, $A \nci_p B | U \cap K$ is in the WTC graphoid closure of the dependence base of $p$ \cite[Theorem 5.1]{Penna2013}. Then, $A \nci_p B \cup S_B | U \cap K$ by decomposition. Moreover, $A \ci_G S_B | U \cap K$ follows from Remark \ref{rem:tree}. Therefore, $A \nci_p B | U \cap K \cup S_B$ by contraction2 and $A \cup S_A \nci_p B | U \cap K \cup S_B$ by decomposition. Moreover, $S_A \ci_G B | U \cap K \cup S_B$ follows from Remark \ref{rem:tree}. Therefore, $A \nci_p B | U \cap K \cup S_B \cup S_A$ by symmetry and contraction2.

Let $D$ be any node in $S$. Then, one of the following cases must occur.

\begin{description}

\item[Case 1] $A \ci_G D | U \cap K \cup S_B \cup S_A$ or $B \ci_G D | U \cap K \cup S_B \cup S_A$. Assume without loss of generality that $A \ci_G D | U \cap K \cup S_B \cup S_A$. Then, $A \nci_p B \cup D | U \cap K \cup S_B \cup S_A$ by decomposition and $A \nci_p B | U \cap K \cup S_B \cup S_A \cup D$ by contradiction2.

\item[Case 2] $A \nci_G D | U \cap K \cup S_B \cup S_A$ and $B \nci_G D | U \cap K \cup S_B \cup S_A$. Then, there are two paths $\rho_{A:D}$ and $\rho_{B:D}$ that are $(U \cap K \cup S_B \cup S_A)$-open. Note that $\rho_{A:D}$ and $\rho_{B:D}$ are of the forms $A \aa \ldots \aa C - \ldots D$ and $B \aa \ldots \aa C - \ldots D$, respectively, by Remark \ref{rem:tree}. Note also that $\rho_{A:C}$ does not contain $B$ because, otherwise, $\rho_{A:D}$ would not be $(U \cap K \cup S_B \cup S_A)$-open since $B$ would be a triplex node in $\rho_{A:D}$ that is not in $U \cap K \cup S_B \cup S_A$. Likewise, $\rho_{B:C}$ does not contain $A$. Now, let $C' \neq C$ denote the closest node to $A$ and $B$ that is in $\rho_{A:C}$ and $\rho_{B:C}$. Then, $\rho_{A:C'} \cup \rho_{C':B}$ is a path which, moreover, is $(U \cap K \cup S_B \cup S_A)$-open. To see the latter, note that $C'$ is a triplex node in both $\rho_{A:D}$ and $\rho_{A:C'} \cup \rho_{C':B}$ and, moreover, it is in $U \cap K \cup S_B \cup S_A$ because $\rho_{A:D}$ is $(U \cap K \cup S_B \cup S_A)$-open. However, this implies that there is a second path between $A$ and $B$ that is $U$-open, which contradicts the assumption that $A \de_G B | U$.

\end{description}

Therefore, by repeating the reasoning above for the rest of the nodes in $S$, we can conclude that $A \nci_p B | U \cap K \cup S_B \cup S_A \cup S$.

Finally, note that $A \ci_G U \setminus K \setminus S_B \setminus S_A \setminus S | U \cap K \cup S_B \cup S_A \cup S$ because there is no path between $A$ and $U \setminus K \setminus S_B \setminus S_A \setminus S$. Then, $A \nci_p B \cup U \setminus K \setminus S_B \setminus S_A \setminus S | U \cap K \cup S_B \cup S_A \cup S$ by decomposition and $A \nci_p B | U$ by contraction2.

Note that the above derivation of $A \nci_p B | U$ only made use of the dependencies in dependence base of $p$ and the nine properties introduced at the beginning of this section. Thus, $A \nci_p B | U$ is in the WTC graphoid closure of the dependence base of $p$.
\end{proof}

\begin{lemma}\label{lem:u}
Let $G$ be a MMCCG of a WTC graphoid $p$. Let $A$ and $B$ denote two nodes that are in the same undirected connectivity component of $G$. If $A \de_G B | U$, then $A \nci_p B | U$ is in the WTC graphoid closure of the dependence base of $p$.
\end{lemma}

\begin{proof}
Let $K$ denote the undirected connectivity component that contains $A$ and $B$. Let $S_A$ denote the nodes in $U \setminus K$ that are in $sp_G(A)$ or connected to $A$ by a path that passes through $sp_G(A)$. Let $S_B$ denote the nodes in $U \setminus K$ that are in $sp_G(B)$ or connected to $B$ by a path that passes through $sp_G(B)$. Let $S$ denote the nodes in $U \setminus K \setminus S_A \setminus S_B$ that are connected to $A$ or $B$ by a path that passes through $ne_G(A)$ or $ne_G(B)$, respectively.

Note that $A \de_G B | U \cap K$. Note also that the path that makes this statement hold only contains undirected edges, because all its nodes are in $K$ by Remark \ref{rem:tree}. Thus, $A \nci_p B | U \cap K$ is in the WTC graphoid closure of the dependence base of $p$ \cite[Theorem 5]{Pennaetal.2009}. Then, $A \nci_p B \cup S_B | U \cap K$ by decomposition. Moreover, $A \ci_G S_B | U \cap K$ follows from Remark \ref{rem:tree}. Therefore, $A \nci_p B | U \cap K \cup S_B$ by contraction2 and $A \cup S_A \nci_p B | U \cap K \cup S_B$ by decomposition. Moreover, $S_A \ci_G B | U \cap K \cup S_B$ follows from Remark \ref{rem:tree}. Therefore, $A \nci_p B | U \cap K \cup S_B \cup S_A$ by symmetry and contraction2.

Let $D$ be any node in $S$. Then, one of the following cases must occur.

\begin{description}

\item[Case 1] $A \ci_G D | U \cap K \cup S_B \cup S_A$ or $B \ci_G D | U \cap K \cup S_B \cup S_A$. Assume without loss of generality that $A \ci_G D | U \cap K \cup S_B \cup S_A$. Then, $A \nci_p B \cup D | U \cap K \cup S_B \cup S_A$ by decomposition and $A \nci_p B | U \cap K \cup S_B \cup S_A \cup D$ by contradiction2.

\item[Case 2] $A \nci_G D | U \cap K \cup S_B \cup S_A$ and $B \nci_G D | U \cap K \cup S_B \cup S_A$. Then, one of the following cases must occur.

\begin{description}

\item[Case 2.1] All the paths between $A$ and $D$ that are $(U \cap K \cup S_B \cup S_A)$-open pass through $B$ or all the paths between $B$ and $D$ that are $(U \cap K \cup S_B \cup S_A)$-open pass through $A$. Assume without loss of generality that all the paths between $A$ and $D$ that are $(U \cap K \cup S_B \cup S_A)$-open pass through $B$. Since $B \notin U \cap K \cup S_B \cup S_A$, $B$ must be a non-triplex node in all these paths. Therefore, none of these paths is $(U \cap K \cup S_B \cup S_A \cup B)$-open because, otherwise, $B$ would have to have some spouse in $G$ and, thus, removing $B$ from any of these paths would result in a path between $A$ and $D$ by definition of MCCGs which, moreover, would be $(U \cap K \cup S_B \cup S_A)$-open and would not pass through $B$, which contradicts the assumption that such a path does not exist. Consequently, $A \ci_G D | U \cap K \cup S_B \cup S_A \cup B$. Then, $A \nci_p B \cup D | U \cap K \cup S_B \cup S_A$ by decomposition and $A \nci_p B | U \cap K \cup S_B \cup S_A \cup D$ by intersection.

\item[Case 2.2] There are two paths $\rho_{A:D}$ and $\rho_{B:D}$ that are $(U \cap K \cup S_B \cup S_A)$-open and such that they do not pass through $B$ and $A$, respectively. Note that $\rho_{A:D}$ and $\rho_{B:D}$ are of the forms $A - \ldots - C \aa \ldots D$ and $B - \ldots - C \aa \ldots D$, respectively, by Remark \ref{rem:tree}. Then, one of the following cases must occur.

\begin{description}

\item[Case 2.2.1] $C$ is the only node that is in $\rho_{A:C}$ and $\rho_{B:C}$. Then, $\rho_{A:C} \cup \rho_{C:B}$ is a path which, moreover, is $(U \cap K \cup S_B \cup S_A)$-open. To see the latter, note that $C$ is a non-triplex node in $\rho_{A:C} \cup \rho_{C:B}$ and it has some spouse in $G$. However, this implies that there is a second path between $A$ and $B$ that is $U$-open, which contradicts the assumption that $A \de_G B | U$.

\item[Case 2.2.2] $C$ is not the only node that is in $\rho_{A:C}$ and $\rho_{B:C}$. Then, let $C' \neq C$ denote the closest node to $A$ and $B$ that is in $\rho_{A:C}$ and $\rho_{B:C}$. Then, $\rho_{A:C'} \cup \rho_{C':B}$ is a path which, moreover, is $(U \cap K \cup S_B \cup S_A)$-open. To see the latter, note that $C'$ is a non-triplex node in both $\rho_{A:D}$ and $\rho_{A:C'} \cup \rho_{C':B}$ and, moreover, it is not in $U \cap K \cup S_B \cup S_A$ because $\rho_{A:D}$ is $(U \cap K \cup S_B \cup S_A)$-open. However, this implies that there is a second path between $A$ and $B$ that is $U$-open, which contradicts the assumption that $A \de_G B | U$.

\end{description}

\end{description}

\end{description}

Therefore, by repeating the reasoning above for the rest of the nodes in $S$, we can conclude that $A \nci_p B | U \cap K \cup S_B \cup S_A \cup S$.

Finally, note that $A \ci_G U \setminus K \setminus S_B \setminus S_A \setminus S | U \cap K \cup S_B \cup S_A \cup S$ because there is no path between $A$ and $U \setminus K \setminus S_B \setminus S_A \setminus S$. Then, $A \nci_p B \cup U \setminus K \setminus S_B \setminus S_A \setminus S | U \cap K \cup S_B \cup S_A \cup S$ by decomposition and $A \nci_p B | U$ by contraction2.

Note that the above derivation of $A \nci_p B | U$ only made use of the dependencies in dependence base of $p$ and the nine properties introduced at the beginning of this section. Thus, $A \nci_p B | U$ is in the WTC graphoid closure of the dependence base of $p$.
\end{proof}

\begin{theorem}\label{the:sound}
Let $G$ be a MMCCG of a WTC graphoid $p$. If $X \de_G Y | Z$, then $X \nci_p Y | Z$ is in the WTC graphoid closure of the dependence base of $p$.
\end{theorem}

\begin{proof}
Let $\rho_{A:B}$ and $U$ denote the path and the set of nodes that make $X \de_G Y | Z$ hold. Then, $A \de_G B | U$. We show below that $A \nci_p B | U$, which implies $X \nci_p Y | Z$ by symmetry, decomposition and weak union.

Let $m$ denote the number of connectivity components $\rho_{A:B}$ passes through. If $m=1$, then the result holds by Lemma \ref{lem:b} or \ref{lem:u}. Assume as induction hypothesis that the result holds for all $m < n$. We now prove it for $m=n$. Let $C$ denote the farthest node from $A$ that is in $\rho_{A:B}$ and in the same connectivity component as $A$. Note that $C \in U$. Note also that $\rho_{A:C}$ and $\rho_{C:B}$ are the only paths between $A$ and $C$ and between $C$ and $B$ that are $(U \setminus C)$-open because, otherwise, there would be a second path between $A$ and $B$ that is $U$-open by Remark \ref{rem:tree}, which contradicts $A \de_G B | U$. Then, $A \de_G C | U \setminus C$ and $C \de_G B | U \setminus C$ and, thus, $A \nci_p C | U \setminus C$ and $C \nci_p B | U \setminus C$ by the induction hypothesis. Note that $A \ci_G B | U \setminus C$ by Remark \ref{rem:tree}. Then, $A \nci_p B | U$ by weak transitivity1.

Note that the above derivation of $X \nci_p Y | Z$ only made use of the dependencies in dependence base of $p$ and the nine properties introduced at the beginning of this section. Thus, $X \nci_p Y | Z$ is in the WTC graphoid closure of the dependence base of $p$.
\end{proof}

\begin{theorem}\label{the:complete}
Let $G$ be a MMCCG of a WTC graphoid $p$. If $X \nci_p Y | Z$ is in the WTC graphoid closure of the dependence base of $p$, then $X \de_G Y | Z$.
\end{theorem}

\begin{proof}
Clearly, all the dependencies in the dependence base of $p$ are identified by the graphical criterion in Definition \ref{def:joined}. Therefore, it only remains to prove that this graphical criterion satisfies the nine properties introduced at the beginning of this section.

\begin{itemize}
\item Symmetry $Y \de_G X | Z \Rightarrow X \de_G Y | Z$. The path $\rho_{A:B}$ and the set of nodes $U$ that make the left-hand side hold also make the right-hand side hold.

\item Decomposition $X \de_G Y | Z \Rightarrow X \de_G Y \cup W | Z$. The path $\rho_{A:B}$ and the set of nodes $U$ that make the left-hand side hold also make the right-hand side hold.

\item Weak union $X \de_G Y | Z \cup W \Rightarrow X \de_G Y \cup W | Z$. The path $\rho_{A:B}$ and the set of nodes $U$ that make the left-hand side hold also make the right-hand side hold.

\item Contraction1 $X \de_G Y \cup W | Z \land X \ci_G Y | Z \cup W \Rightarrow X \de_G W | Z$. Let $\rho_{A:B}$ and $U$ denote the path and the set of nodes that make the left-hand side hold. Following Remark \ref{rem:noinpath}, we can assume without loss of generality that $A$ and $B$ are the only nodes in $\rho_{A:B}$ that are in $X$ and $Y W$, respectively. Note also that $X \ci_G Y | Z W$ implies that no node in the path $\rho_{A:B}$ can be in $Y$. Then, $\rho_{A:B}$ is $(U \setminus Y)$-open. If there is a second path between $A$ and $B$ in $G$ that is $(U \setminus Y)$-open, then let $\varrho_{A:B}$ be any of the shortest such paths. Then, we can find a node $C \in W \setminus U$ such that $\rho_{A:B}$ is $(U \setminus Y \cup C)$-open but $\varrho_{A:B}$ is not. To see it, note that if $\varrho_{A:B}$ is $(U \setminus Y)$-open, then it must contain a non-triplex node $D \in U \cap Y$ because, otherwise, $\varrho_{A:B}$ would be $U$-open, which is a contradiction. Moreover, note that $X \ci_G Y | Z \cup W$ implies that one of the following cases must occur.

\begin{description}

\item[Case 1] $\varrho_{A:D}$ contains a triplex node that is not in $Z$ or $W$. However, this contradicts the assumption that $\varrho_{A:B}$ is $(U \setminus Y)$-open.

\item[Case 2] $\varrho_{A:D}$ contains a non-triplex node that is in $Z$ or $U \cap W$. However, this contradicts the assumption that $\varrho_{A:B}$ is $(U \setminus Y)$-open. 

\item[Case 3] Cases 1 and 2 do not apply. Then, $\varrho_{A:D}$ must contain a non-triplex node $C \in W \setminus U$. Clearly, $\rho_{A:B}$ is $(U \setminus Y \cup C)$-open. Moreover, adding $C$ to $U \setminus Y$ does not activate new paths. That is, if a path $\varphi$ between two nodes in $G$ is not $(U \setminus Y)$-open, then it is not $(U \setminus Y \cup C)$-open because, otherwise, $C$ would have to be a triplex node in $\varphi$ and, thus, $C$ would have some spouse in $G$ and, thus, removing $C$ from $\varrho_{A:B}$ would result in a path between $A$ and $B$ in $G$ by definition of MCCGs which, moreover, would be $(U \setminus Y)$-open, which contradicts the assumption that $\varrho_{A:B}$ is one of the shortest such paths.

\end{description}

Therefore, by repeating the reasoning above we can obtain a set of nodes $U'$ such that $Z \subseteq U' \subseteq X \cup W \cup Z \setminus \{A,B\}$ and $\rho_{A:B}$ is the only path between $A$ and $B$ in $G$ that is $U'$-open. Consequently, $X \de_G W | Z$ holds.

\item Contraction2 $X \de_G Y \cup W | Z \land X \ci_G W | Z \Rightarrow X \de_G Y | Z \cup W$. Let $\rho_{A:B}$ and $U$ denote the path and the set of nodes that make the left-hand side hold. Following Remark \ref{rem:noinpath}, we can assume without loss of generality that $A$ and $B$ are the only nodes in $\rho_{A:B}$ that are in $X$ and $Y W$, respectively. Note also that $X \ci_G W | Z$ implies that no node in the path $\rho_{A:B}$ can be in $W$. Then, $\rho_{A:B}$ is $(U \cup W)$-open. If there is a second path $\varrho_{A:B}$ between $A$ and $B$ in $G$ that is $(U \cup W)$-open, then we can find a node $C \in U \cap Y$ such that $\rho_{A:B}$ is $(U \cup W \setminus C)$-open but $\varrho_{A:B}$ is not. To see it, note that if $\varrho_{A:B}$ is $(U \cup W)$-open, then it must contain a triplex node $D \in W \setminus U$ because, otherwise, $\varrho_{A:B}$ would be $U$-open, which is a contradiction. Moreover, note that $X \ci_G W | Z$ implies that one of the following cases must occur.

\begin{description}

\item[Case 1] $\varrho_{A:D}$ contains a non-triplex node that is in $Z$. However, this contradicts the assumption that $\varrho_{A:B}$ is $(U \cup W)$-open.

\item[Case 2] $\varrho_{A:D}$ contains a triplex node that is not in $Z$ or $Y \setminus U$. However, this contradicts the assumption that $\varrho_{A:B}$ is $(U \cup W)$-open. 

\item[Case 3] Cases 1 and 2 do not apply. Then, $\varrho_{A:D}$ must contain a triplex node $C \in U \cap Y$. Clearly, $\rho_{A:B}$ is $(U \cup W \setminus C)$-open. Moreover, removing $C$ from $U \cup W$ does not activate new paths. That is, if a path $\varphi$ between two nodes in $G$ is not $(U \cup W)$-open, then it is not $(U \cup W \setminus C)$-open because, otherwise, $C$ would have to be a non-triplex node in $\varphi$. However, recall that $C$ is a triplex node in $\varrho_{A:D}$. Then, $C$ has some spouse in $G$ and, thus, $\varphi$ would be $(U \cup W)$-open, which is a contradiction.

\end{description}

Therefore, by repeating the reasoning above we can obtain a set of nodes $U'$ such that $Z \cup W \subseteq U' \subseteq X \cup Y \cup W \cup Z \setminus \{A,B\}$ and $\rho_{A:B}$ is the only path between $A$ and $B$ in $G$ that is $U'$-open. Consequently, $X \de_G Y | Z \cup W$ holds.

\item Intersection $X \de_G Y \cup W | Z \land X \ci_G Y | Z \cup W \Rightarrow X \de_G W | Z \cup Y$. Let $\rho_{A:B}$ and $U$ denote the path and the set of nodes that make the left-hand side hold. Following Remark \ref{rem:noinpath}, we can assume without loss of generality that $A$ and $B$ are the only nodes in $\rho_{A:B}$ that are in $X$ and $Y W$, respectively. Note also that $X \ci_G Y | Z \cup W$ implies that no node in the path $\rho_{A:B}$ can be in $Y$. Then, $\rho_{A:B}$ is $(U \cup Y)$-open. If there is a second path between $A$ and $B$ in $G$ that is $(U \cup Y)$-open, then let $\varrho_{A:B}$ be any of the shortest such paths. Then, we can find a node $C \in W \setminus U$ such that $\rho_{A:B}$ is $(U \cup Y \cup C)$-open but $\varrho_{A:B}$ is not. To see it, note that if $\varrho_{A:B}$ is $(U \cup Y)$-open, then it must contain a triplex node $D \in Y \setminus U$ because, otherwise, $\varrho_{A:B}$ would be $U$-open, which is a contradiction. Moreover, note that $X \ci_G Y | Z \cup W$ implies that one of the following cases must occur.

\begin{description}

\item[Case 1] $\varrho_{A:D}$ contains a triplex node that is not in $Z$ or $W$. However, this contradicts the assumption that $\varrho_{A:B}$ is $(U \cup Y)$-open.

\item[Case 2] $\varrho_{A:D}$ contains a non-triplex node that is in $Z$ or $U \cap W$. However, this contradicts the assumption that $\varrho_{A:B}$ is $(U \cup Y)$-open. 

\item[Case 3] Cases 1 and 2 do not apply. Then, $\varrho_{A:D}$ must contain a non-triplex node $C \in W \setminus U$. Clearly, $\rho_{A:B}$ is $(U \cup Y \cup C)$-open. Moreover, adding $C$ to $U \cup Y$ does not activate new paths. That is, if a path $\varphi$ between two nodes in $G$ is not $(U \cup Y)$-open, then it is not $(U \cup Y \cup C)$-open because, otherwise, $C$ would have to be a triplex node in $\varphi$ and, thus, $C$ would have some spouse in $G$ and, thus, removing $C$ from $\varrho_{A:B}$ would result in a path between $A$ and $B$ in $G$ by definition of MCCGs which, moreover, would be $(U \cup Y)$-open, which contradicts the assumption that $\varrho_{A:B}$ is one of the shortest such paths.

\end{description}

Therefore, by repeating the reasoning above we can obtain a set of nodes $U'$ such that $Z \cup Y \subseteq U' \subseteq X \cup Y \cup W \cup Z \setminus \{A,B\}$ and $\rho_{A:B}$ is the only path between $A$ and $B$ in $G$ that is $U'$-open. Consequently, $X \de_G W | Z \cup Y$ holds.

\item Weak transitivity1 $X \de_G K | Z \land K \de_G Y | Z \land X \ci_G Y | Z \Rightarrow X \de_G Y | Z \cup K$. Let $\rho_{A:K}$ and $U$ denote the path and the set of nodes that make $X \de_G K | Z$ hold. Likewise, let $\rho_{K:B}$ and $W$ denote the path and the set of nodes that make $K \de_G Y | Z$ hold. We show below that the path $\rho_{A:K} \cup \rho_{K:B}$ and the set of nodes $U \cup W \cup K$ make $X \de_G Y | Z \cup K$ hold. Following Remark \ref{rem:noinpath}, we can assume without loss of generality that $A$ is the only node in $\rho_{A:K}$ that is in $X$, and that $B$ is the only node in $\rho_{K:B}$ that is in $Y$. Note also that $X \ci_G Y | Z$ implies that $\rho_{A:K}$ has no node in $Y$ and $\rho_{K:B}$ has no node in $X$. Following Remark \ref{rem:nospouse}, we assume without loss of generality that the nodes in $U$ that are not in $Z$ or $\rho_{A:K}$ have no spouse in $G$, and that the nodes in $W$ that are not in $Z$ or $\rho_{K:B}$ have no spouse in $G$.

First, note that $\rho_{A:K}$ is the only path between $A$ and $K$ in $G$ that is $(U \cup W)$-open. To see it, note that the nodes that are in both $W$ and $\rho_{K:B}$ are also in $Z$ and, thus, in $U$. On the other hand, the nodes that are in $W$ but not in $Z$ or $\rho_{K:B}$ do not have any spouse in $G$ and, thus, they cannot activate any new path between $A$ and $K$ in $G$. Likewise, $\rho_{K:B}$ is the only path between $K$ and $B$ in $G$ that is $(U \cup W)$-open.

Second, note that $\rho_{A:K} \cup \rho_{K:B}$ is a path, because $K$ is the only node that is in both $\rho_{A:K}$ and $\rho_{K:B}$. To see it, assume the contrary. Specifically, let $C \neq K$ denote the closest node to $A$ and $B$ that is in both $\rho_{A:K}$ and $\rho_{K:B}$. Note that the path $\rho_{A:C} \cup \rho_{C:B}$ cannot be $(U \cup W)$-open by $X \ci_G Y | Z$. Therefore, one of the following cases must occur.

\begin{description}

\item[Case 1] $C$ is a non-triplex node in $\rho_{A:C} \cup \rho_{C:B}$ and $C \in Z$. However, that $C \in Z$ together with the fact that $\rho_{A:K}$ is $(U \cup W)$-open imply that $C$ is a triplex node in $\rho_{A:K}$. Thus, $C$ has some spouse in $G$ and, thus, $\rho_{A:C} \cup \rho_{C:B}$ is $(U \cup W)$-open, which is a contradiction.

\item[Case 2] $C$ is a triplex node in $\rho_{A:C} \cup \rho_{C:B}$ and $C \notin Z$. However, that $C \notin Z$ together with the fact that $\rho_{A:K}$ is $(U \cup W)$-open imply that $C$ is a non-triplex node in $\rho_{A:K}$. Moreover, that $C$ is a triplex node in $\rho_{A:C} \cup \rho_{C:B}$ implies that $C$ has some spouse in $G$. Then, removing $C$ from $\rho_{A:K}$ results in a second path between $A$ and $K$ by definition of MCCGs which, moreover, is $(U \cup W)$-open, which is a contradiction.

\end{description}

Moreover, note that $\rho_{A:K} \cup \rho_{K:B}$ must be $(U \cup W \cup K)$-open because, otherwise, $K$ would have to be a non-triplex node in $\rho_{A:K} \cup \rho_{K:B}$, which would contradict $X \ci_G Y | Z$.

Finally, if there is a second path between $A$ and $B$ in $G$ that is $(U \cup W \cup K)$-open, then $K$ must be a triplex node in that path because, otherwise, that path would contradict $X \ci_G Y | Z$. However, this implies that there is a second path between $A$ and $K$ or between $K$ and $B$ in $G$ that is $(U \cup W)$-open, which is a contradiction.

\item Weak transitivity2 $X \de_G K | Z \land K \de_G Y | Z \land X \ci_G Y | Z \cup K \Rightarrow X \de_G Y | Z$. Let $\rho_{A:K}$ and $U$ denote the path and the set of nodes that make $X \de_G K | Z$ hold. Likewise, let $\rho_{K:B}$ and $W$ denote the path and the set of nodes that make $K \de_G Y | Z$ hold. We show below that the path $\rho_{A:K} \cup \rho_{K:B}$ and the set of nodes $U \cup W$ make $X \de_G Y | Z$ hold. Following Remark \ref{rem:noinpath}, we can assume without loss of generality that $A$ is the only node in $\rho_{A:K}$ that is in $X$, and that $B$ is the only node in $\rho_{K:B}$ that is in $Y$. Note also that $X \ci_G Y | Z \cup K$ implies that $\rho_{A:K}$ has no node in $Y$ and $\rho_{K:B}$ has no node in $X$. Following Remark \ref{rem:nospouse}, we assume without loss of generality that the nodes in $U$ that are not in $Z$ or $\rho_{A:K}$ have no spouse in $G$, and that the nodes in $W$ that are not in $Z$ or $\rho_{K:B}$ have no spouse in $G$.

First, note that $\rho_{A:K}$ is the only path between $A$ and $K$ in $G$ that is $(U \cup W)$-open, and that $\rho_{K:B}$ is the only path between $K$ and $B$ in $G$ that is $(U \cup W)$-open. To see it, repeat the reasoning in weak transitivity1.

Second, note that $\rho_{A:K} \cup \rho_{K:B}$ is a path between $A$ and $B$ in $G$. To see it, repeat the reasoning in weak transitivity1 (note that $X \ci_G Y | Z \cup K$ should be used instead of $X \ci_G Y | Z$).

Moreover, note that $\rho_{A:K} \cup \rho_{K:B}$ must be $(U \cup W)$-open because, otherwise, $K$ would have to be a triplex node in $\rho_{A:K} \cup \rho_{K:B}$, which would contradict $X \ci_G Y | Z \cup K$.

Finally, if there is a second path between $A$ and $B$ in $G$ that is $(U \cup W)$-open, then $K$ must be a non-triplex node in that path because, otherwise, that path would contradict $X \ci_G Y | Z \cup K$. However, this implies that there is a second path between $A$ and $K$ or between $K$ and $B$ in $G$ that is $(U \cup W)$-open, which is a contradiction.

\item Composition $X \de_G Y \cup W | Z \land X \ci_G Y | Z \Rightarrow X \de_G W | Z$. Let $\rho_{A:B}$ and $U$ denote the path and the set of nodes that make the left-hand side hold. Following Remark \ref{rem:noinpath}, we can assume without loss of generality that $A$ and $B$ are the only nodes in $\rho_{A:B}$ that are in $X$ and $Y W$, respectively. Note also that $X \ci_G Y | Z$ implies that no node in the path $\rho_{A:B}$ can be in $Y$. Then, $\rho_{A:B}$ is $(U \setminus Y)$-open. If there is a second path $\varrho_{A:B}$ between $A$ and $B$ in $G$ that is $(U \setminus Y)$-open, then we can find a node $C \in U \setminus Y$ such that $\rho_{A:B}$ is $(U \setminus Y \setminus C)$-open but $\varrho_{A:B}$ is not. To see it, note that if $\varrho_{A:B}$ is $(U \setminus Y)$-open, then it must contain a non-triplex node $D \in U \cap Y$ because, otherwise, $\varrho_{A:B}$ would be $U$-open, which is a contradiction. Moreover, note that $X \ci_G Y | Z$ implies that one of the following cases must occur.

\begin{description}

\item[Case 1] $\varrho_{A:D}$ contains a non-triplex node that is in $Z$. However, this contradicts the assumption that $\varrho_{A:B}$ is $(U \setminus Y)$-open.

\item[Case 2] $\varrho_{A:D}$ contains a triplex node that is not in $Z$ or $W \setminus U$. However, this contradicts the assumption that $\varrho_{A:B}$ is $(U \setminus Y)$-open. 

\item[Case 3] Cases 1 and 2 do not apply. Then, $\varrho_{A:D}$ must contain a triplex node $C \in U \cap W$. Clearly, $\rho_{A:B}$ is $(U \setminus Y \setminus C)$-open. Moreover, removing $C$ from $U \setminus Y$ does not activate new paths. That is, if a path $\varphi$ between two nodes in $G$ is not $(U \setminus Y)$-open, then it is not $(U \setminus Y \setminus C)$-open because, otherwise, $C$ would have to be a non-triplex node in $\varphi$. However, recall that $C$ is a triplex node in $\varrho_{A:D}$. Then, $C$ has some spouse in $G$ and, thus, $\varphi$ would be $(U \setminus Y)$-open, which is a contradiction.

\end{description}

Therefore, by repeating the reasoning above we can obtain a set of nodes $U'$ such that $Z \subseteq U' \subseteq X \cup W \cup Z \setminus \{A,B\}$ and $\rho_{A:B}$ is the only path between $A$ and $B$ in $G$ that is $U'$-open. Consequently, $X \de_G W | Z$ holds.

\end{itemize}
\end{proof}

While Theorem \ref{the:sound} may be somewhat expected because if there is a single path between $A$ and $B$ in $G$ that is $U$-open then there is no possibility of path cancelation, the combination of Theorems \ref{the:sound} and \ref{the:complete} is rather exciting: We now have a simple graphical criterion to decide whether a given dependence is or is not in the WTC graphoid closure of the dependence base of $p$, i.e. we do not need to try to find a derivation of it, which is usually a tedious task.

\begin{corollary}
Let $G$ be a MMCCG of a WTC graphoid $p$. Then, $X \de_G Y | Z$ if and only if $X \nci_G Y | Z$ is in the WTC graphoid closure of the dependence base of $p$.
\end{corollary}

It is worth mentioning that the graphical criterion in Definition \ref{def:joined} is not complete in the sense of identifying all the dependencies that are shared by all the WTC graphoids whose MMCCG is $G$. Note also that neither the graphical criterion in Definition \ref{def:joined} nor any other sound graphical criterion can be complete in the sense of identifying all the dependencies in $p$. See \cite[pp. 1082-1083]{Pennaetal.2009} and \cite[pp. 202-203]{Penna2013} for counterexamples.

One of the reasons for developing the graphical criterion in Definition \ref{def:joined} is that $X \nci_G Y | Z$ does not imply $X \nci_p Y | Z$. However, if $G$ has no cycle, then the corollary below proves that $X \nci_G Y | Z$ does imply $X \nci_p Y | Z$ and, moreover, that this way of identifying dependencies in $p$ is sound and complete in the strictest sense possible, since all and only all of them are identified.

\begin{corollary}
Let $G$ be a MMCCG of a WTC graphoid $p$. If $G$ has no cycle, then $p$ is faithful to $G$.
\end{corollary}

\begin{proof}
Assume to the contrary that $p$ is not faithful to $G$. Since $G$ is a MCCG of $p$, this assumption is equivalent to assume that there exist three pairwise disjoint subsets of $V$, here denoted as $X$, $Y$ and $Z$, such that $X \nci_G Y | Z$ but $X \ci_p Y | Z$. However, $X \nci_G Y | Z$ implies that there must exist a path in $G$ between some node $A \in X$ and some node $B \in Y$ that is $(X \cup Y \cup Z \setminus \{A,B\})$-open. Furthermore, since $G$ has no cycle, that must be the only such path between $A$ and $B$ in $G$. However, this implies $X \de_G Y | Z$ and thus $X \nci_p Y | Z$ by Theorem \ref{the:sound}, which is a contradiction.
\end{proof}

\subsection{Discussion}\label{sec:discussion3}

In this section, we have introduced a sound and complete graphical criterion for reading dependencies from a MCCG of a WTC graphoid, e.g. a Gaussian probability distribution. Recall that one of the advantages of MCCGs is the ability to model the covariance and concentration matrices of a Gaussian probability distribution jointly with a single graph, rather than modeling the former with a covariance graph and the latter with a concentration graph. We have argued that, by doing so, MCCGs may model more accurately the probability distribution. We show below two examples that illustrate this. Specifically, the examples show that a MMCCG of a WTC graphoid $p$ can identify more (in)dependencies in $p$ than the covariance graph and the concentration graph of $p$ jointly.

\begin{example}
Consider a Gaussian probability distribution $p$ that is faithful to the MCCG $G$ below. Recall from Theorem \ref{the:faithfulness} that such a probability distribution exists. Note that $G$ is a MMCCG of $p$.

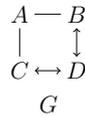
\begin{table}[H]
\centering
\scalebox{0.75}{
\begin{tabular}{c}
\begin{tikzpicture}[inner sep=1mm]
\node at (0,0) (A) {$A$};
\node at (1,0) (B) {$B$};
\node at (0,-1) (C) {$C$};
\node at (1,-1) (D) {$D$};
\path[-] (A) edge (B);
\path[-] (A) edge (C);
\path[<->] (B) edge (D);
\path[<->] (C) edge (D);
\end{tikzpicture}\\
$G$
\end{tabular}}
\end{table}

The covariance graph and the concentration graph of $p$ are depicted by the graphs $H$ and $F$ below.

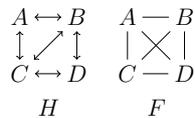
\begin{table}[H]
\centering
\scalebox{0.75}{
\begin{tabular}{cc}
\begin{tikzpicture}[inner sep=1mm]
\node at (0,0) (A) {$A$};
\node at (1,0) (B) {$B$};
\node at (0,-1) (C) {$C$};
\node at (1,-1) (D) {$D$};
\path[<->] (A) edge (B);
\path[<->] (A) edge (C);
\path[<->] (B) edge (C);
\path[<->] (B) edge (D);
\path[<->] (C) edge (D);
\end{tikzpicture}
&
\begin{tikzpicture}[inner sep=1mm]
\node at (0,0) (A) {$A$};
\node at (1,0) (B) {$B$};
\node at (0,-1) (C) {$C$};
\node at (1,-1) (D) {$D$};
\path[-] (A) edge (B);
\path[-] (A) edge (C);
\path[-] (A) edge (D);
\path[-] (B) edge (C);
\path[-] (B) edge (D);
\path[-] (C) edge (D);
\end{tikzpicture}\\
$H$ & $F$
\end{tabular}}
\end{table}

Now, note that $B \ci_p C | A$ because $B \ci_G C | A$. However, $B \nci_H C | A$ and $B \nci_F C | A$.
\end{example}

\begin{example}
Consider a Gaussian probability distribution $p$ that is faithful to the MCCG $G$ below. Recall from Theorem \ref{the:faithfulness} that such a probability distribution exists. Note that $G$ is a MMCCG of $p$ and, moreover, that it has no cycle with both undirected and bidirected edges.

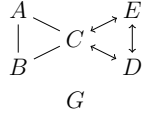
\begin{table}[H]
\centering
\scalebox{0.75}{
\begin{tabular}{c}
\begin{tikzpicture}[inner sep=1mm]
\node at (0,0) (A) {$A$};
\node at (0,-1) (B) {$B$};
\node at (1,-0.5) (C) {$C$};
\node at (2,-1) (D) {$D$};
\node at (2,0) (E) {$E$};
\path[-] (A) edge (B);
\path[-] (A) edge (C);
\path[-] (B) edge (C);
\path[<->] (C) edge (D);
\path[<->] (C) edge (E);
\path[<->] (D) edge (E);
\end{tikzpicture}\\
$G$
\end{tabular}}
\end{table}

The covariance graph and the concentration graph of $p$ are depicted by the graphs $H$ and $F$ below.

\begin{table}[H]
\centering
\scalebox{0.75}{
\begin{tabular}{cc}
\begin{tikzpicture}[inner sep=1mm]
\node at (0,0) (A) {$A$};
\node at (0,-1) (B) {$B$};
\node at (1,-0.5) (C) {$C$};
\node at (2,-1) (D) {$D$};
\node at (2,0) (E) {$E$};
\path[<->] (A) edge (B);
\path[<->] (A) edge (C);
\path[<->] (B) edge (C);
\path[<->] (C) edge (D);
\path[<->] (C) edge (E);
\path[<->] (D) edge (E);
\end{tikzpicture}
&
\begin{tikzpicture}[inner sep=1mm]
\node at (0,0) (A) {$A$};
\node at (0,-1) (B) {$B$};
\node at (1,-0.5) (C) {$C$};
\node at (2,-1) (D) {$D$};
\node at (2,0) (E) {$E$};
\path[-] (A) edge (B);
\path[-] (A) edge [bend left] (E);
\path[-] (A) edge [bend left] (D);
\path[-] (A) edge (C);
\path[-] (B) edge (C);
\path[-] (B) edge [bend right] (D);
\path[-] (B) edge [bend right] (E);
\path[-] (C) edge (D);
\path[-] (C) edge (E);
\path[-] (D) edge (E);
\end{tikzpicture}\\
$H$ & $F$
\end{tabular}}
\end{table}

Now, note that $A \nci_p D | B C$ because $A \de_G D | B C$. However, neither $A \de_H D | B C$ nor $A \de_F D | B C$ hold.
\end{example}

Despite the examples above, we do not discard the possibility that some Gaussian probability distributions are modeled more accurately by a covariance graph plus a concentration graph than by a MCCG. We would like to study when this occurs, if at all.

\section*{Acknowledgments}

We would like to thank the anonymous Reviewers and specially Reviewer 1 for their comments. This work is funded by the Center for Industrial Information Technology (CENIIT) and a so-called career contract at Link\"oping University, by the Swedish Research Council (ref. 2010-4808), and by FEDER funds and the Spanish Government (MICINN) through the project TIN2010-20900-C04-03.

\end{document}